\documentclass[nohyperref]{article}

\usepackage{microtype}
\usepackage{graphicx}
\usepackage{subcaption}
\usepackage{booktabs} % for professional tables
\usepackage{hyperref}

\usepackage[accepted]{icml2022}

\usepackage{amsmath}
\usepackage{amssymb}
\usepackage{mathtools}
\usepackage{amsthm}
\usepackage{makecell}
\usepackage{multirow}
\usepackage{colortbl}
\definecolor{strings}{rgb}{.824,.251,.259}
\definecolor{keywords}{rgb}{.224,.451,.686}
\definecolor{comment}{rgb}{.322,.451,.322}
\definecolor{LightCyan}{rgb}{0.88,1,1}

\newcommand{\R}{\mathbb{R}}
\newcommand{\mathbold}[1]{\ensuremath{\boldsymbol{\mathbf{#1}}}}

\newcommand{\mcN}{\mathcal{N}}
\newcommand{\mcO}{\mathcal{O}}
\newcommand{\mba}{\mathbold{a}}

\newcommand{\mbk}{\mathbold{k}}

\newcommand{\mbq}{\mathbold{q}}

\newcommand{\mbv}{\mathbold{v}}
\newcommand{\mbw}{\mathbold{w}}
\newcommand{\mbx}{\mathbold{x}}
\newcommand{\mby}{\mathbold{y}}

\newcommand{\mbK}{\mathbold{K}}

\newcommand{\mbQ}{\mathbold{Q}}

\newcommand{\mbV}{\mathbold{V}}

\newcommand{\mbY}{\mathbold{Y}}

\newcommand{\mbmu}{\mathbold{\mu}}

\newcommand{\mbphi}{\mathbold{\phi}}
\newcommand{\mbpi}{\mathbold{\pi}}

\newcommand{\E}{\mathbb{E}}

\newcommand\numberthis{\addtocounter{equation}{1}\tag{\theequation}}

\DeclareMathOperator{\Var}{Var}
\DeclareMathOperator{\Cov}{Cov}
\DeclarePairedDelimiter\norm{\lVert}{\rVert}

\usepackage{xspace}
\xspaceaddexceptions{]\}}
\newcommand{\imagenet}{\texttt{ImageNet1k}\xspace}
\newcommand{\imagenetfull}{\texttt{ImageNet21k}\xspace}
\newcommand{\kinetics}{\texttt{K400}\xspace}
\newcommand{\ssv}{\texttt{SSv2}\xspace}
\newcommand{\wmt}{\texttt{WMT}\xspace}
\newcommand{\lra}{\texttt{LRA}\xspace}
\usepackage[capitalize,noabbrev]{cleveref}
\crefname{section}{§}{§§}
\Crefname{section}{§}{§§}
\creflabelformat{equation}{#2\textup{#1}#3}
\theoremstyle{plain}
\newtheorem{theorem}{Theorem}[section]
\newtheorem{proposition}[theorem]{Proposition}
\newtheorem{lemma}[theorem]{Lemma}
\newtheorem{corollary}[theorem]{Corollary}
\theoremstyle{definition}

\theoremstyle{remark}

\usepackage[textsize=tiny]{todonotes}

\icmltitlerunning{Linear Complexity Randomized Self-attention Mechanism}

\begin{document}

\twocolumn[
\icmltitle{Linear Complexity Randomized Self-attention Mechanism}



\icmlsetsymbol{intern}{*}

\begin{icmlauthorlist}
\icmlauthor{Lin Zheng}{univ,intern}
\icmlauthor{Chong Wang}{comp}
\icmlauthor{Lingpeng Kong}{univ,lab}
\end{icmlauthorlist}

\icmlaffiliation{univ}{Department of Computer Science, The University of Hong Kong}
\icmlaffiliation{comp}{ByteDance Inc.}
\icmlaffiliation{lab}{Shanghai Artificial Intelligence Laboratory}

\icmlcorrespondingauthor{Lin Zheng}{linzheng@connect.hku.hk}

\icmlkeywords{Machine Learning, Attention, Deep Learning}

\vskip 0.3in
]



\printAffiliationsAndNotice{\textsuperscript{*} The majority of this work was done while the first author was interning at Bytedance.}
%%%%%%%%%%%%%%%%%%%%%%%%%%%%%%%%%%%%%%%%%%%%%%%%%%%%%%%%%%%%%%%%%%%%%%%%%%%%%%%%%%%%%%%%%%%
\begin{abstract}
Recently, random feature attentions (RFAs) are proposed to approximate the softmax attention in linear time and space complexity by linearizing the exponential kernel. In this paper, we first propose a novel perspective to understand the bias in such approximation by recasting RFAs as self-normalized importance samplers. This perspective further sheds light on an \emph{unbiased} estimator for the whole softmax attention, called randomized attention (RA). RA constructs positive random features via query-specific distributions and enjoys greatly improved approximation fidelity, albeit exhibiting quadratic complexity. By combining the expressiveness in RA and the efficiency in RFA, we develop a novel linear complexity self-attention mechanism called linear randomized attention (LARA). Extensive experiments across various domains demonstrate that RA and LARA significantly improve the performance of RFAs by a substantial margin.
\end{abstract}

\section{Introduction}
\label{sec:intro}
Transformers \citep{vaswani2017attention} are powerful neural networks for sequence modeling. They have been successfully applied in various domains, such as natural language processing \citep{vaswani2017attention,dehghani2018universal,devlin-etal-2019-bert,raffel2020t5}, computer vision \citep{carion2020detr,dosovitskiy2021vit,liu2021swin}, bioinformatics \citep{Rivese2016239118,jumper2021highly} and reinforcement learning \citep{chen2021decision}. The core building block of transformer models is the self-attention mechanism, which captures complex interactions among sequence elements \citep{vaswani2017attention}. 

However, the computational complexity of attention mechanism is quadratic in the number of tokens, making it prohibitive to process long sequences. In the past two years, there has been a community effort towards developing efficient attention architectures with improved computation complexity and memory usage \citep{tay2020efficient}. 
% many efficient attention variants are proposed, including those relying on sparse attention patterns \citep[\emph{inter alia}]{child2019generating, Kitaev2020reformer, beltagy2020longformer}, low-rank approximations \citep[\emph{inter alia}]{xiong2021nystromformer, katharopoulos2020transformers_are_rnns,choromanski2021rethinking, peng2021rfa} and so on. 
Among them, one prominent is to view the attention mechanism through kernelization \citep[\emph{inter alia}]{katharopoulos2020transformers_are_rnns,choromanski2021rethinking,peng2021rfa}. In this work, we focus on random feature attentions (RFAs) \citep{peng2021rfa, choromanski2021rethinking}, which approximate softmax attention by linearizing the exponential kernel into a dot product of random feature maps. Despite achieving linear time and space complexity, this approximation is biased to the softmax attention as a whole.\footnote{There are several variants of random feature maps that yield an unbiased estimate of the exponential kernel \citep{peng2021rfa, choromanski2021rethinking}. Nevertheless, RFAs still run a biased approximation to the whole softmax attention, since the softmax attention involves a ratio of these exponential kernels. Although the estimator is still consistent, the bias in question is elusive and might impair the approximation fidelity of random features.}

% Such approximation achieves linear time and space complexity due to its efficient computation reuse of key-value statistics among different queries \citep{katharopoulos2020transformers_are_rnns}.
% allows the attention function to be executed in linear, greatly scaling transformer models to much longer sequences.

In this work, we revisit RFA and show that it can be reinterpreted as a self-normalized importance sampler to softmax attention. This insight reveals that the source of the approximation bias in RFAs comes from the self-normalization in estimation \citep{mcbook}. We further show softmax attention can be written as an expectation of linearized attention over an input-dependent mixture distribution. These findings suggest that we can in principle construct an unbiased estimator for the softmax attention as a whole, as opposed to merely exponential kernels in previous work. We call such unbiased estimation \emph{randomized attention} or RA. To the best of our knowledge, this is the first unbiased approximation of the whole softmax attention via kernel linearization.

RA constructs positive random features via distributions exclusive to each query. Since RFAs only employ an input-agnostic standard Gaussian as the importance sampling proposal, RA enables a finer-grained treatment for query-specific information and greatly improves the approximation fidelity;
however, it is as expensive as softmax attention computationally with quadratic complexity, because the key-value statistics are different for each query, unlike the ones in RFAs.
% , which adopts a single standard Gaussian proposal and shares the same set of samples for all queries. 
% As a consequence, computations involving key and value vectors are allowed to be reused among queries, albeit at the cost of greatly limited representational power.
% performs an importance sampling estimate in our interpretation, 

% Firstly, sampling from its distribution scales quadratically; secondly, since its distribution is query-dependent, the computations for each query cannot be reused among different queries as in RFA, also resulting in quadratic complexity; thirdly, evaluating the probability density involves quadratic computations. 

Based on the analysis, one question naturally arises: \emph{``Can we combine the expressiveness in RA and the efficiency in RFA to get the best of both worlds?''}
%  we propose an improved importance sampling estimator to approximate the softmax attention, combining both the expressiveness of RA and the efficiency of RFA. 
To achieve that, we generalize the importance sampling formulation of RFA by adopting \emph{multiple proposals}, each of which depends on different subsets of queries.
% This strategy allows the model to capture finer-grained query information and improves modeling flexibility. 
We further apply multiple importance sampling~\citep{veach1995optimally} and put together these proposals to approximate softmax attention adaptively for different queries, retaining the query-specific property of RA. Meanwhile, since these proposals are shared among all queries, we inherit the efficient computation reuse in RFA and achieve linear complexity. 
We refer to this efficient attention mechanism as LineAr Randomized Attention (LARA). 
Extensive experiments and analyses demonstrate that RA, as well as its linear variant LARA, significantly reduce the approximation error of RFAs. They improve RFAs by a substantial margin across various tasks, including image classification, video action recognition, machine translation, and so on, while retaining computational efficiency.

\section{Background}
\subsection{Softmax Attention}
\label{ssec:attn}
% \footnote{We assume all the queries, keys and values are of the same dimension to simplify the notations, although in practice the dimension of value vectors may be different from queries and keys.}
Let $\mbQ \in \R^{N \times D}$, $\mbK \in \R^{M \times D}$ and $\mbV \in \R^{M \times D}$ denote the sets of $N$ query vectors, $M$ key and value vectors respectively. For each query $\mbq_n$, the softmax attention computes the following quantity,\footnote{We omit the commonly used scaling factor $1 / \sqrt{d}$ for simplicity as it can be merged into the computation of queries or keys.}
\begin{align*}
  \mathsf{SoftmaxAttn}\left(\mbq_{n},\mbK,\mbV\right)\!&\coloneqq\!\sum_{m=1}^M\!\frac{\exp\!\left(\mbq_{n}^\top \mbk_{m} \right)}{\sum_{m'=1}^M \exp\!\left(\mbq_{n}^\top \mbk_{m'} \right)} \mbv_{m}^{\top}.
%   \numberthis{}\label{eqn:attn}
\end{align*}
Intuitively, the softmax attention first computes the normalized similarity between the query and each key, which is then used to weight value vectors. In the case of self-attention in Transformers \citep{vaswani2017attention}, we have $N = M$; as a result, such mechanism suffers from quadratic time and memory complexity due to the explicit computation of the similarity scores between all pairs of queries and keys.
% Although softmax attention can capture complex long range dependencies, it has to explicitly compute the similarity between all pairs of queries and keys, incurring quadratic complexity in both time and memory.

\subsection{Random Feature Attention}
\label{ssec:rfa}
To reduce the computational complexity of softmax attention, recent work \citep{choromanski2021rethinking, peng2021rfa} proposes to linearize exponential kernels via random feature methods \citep{random-features}. According to Bochner's theorem \citep{bochner2020harmonic}, they re-write the exponential kernel $\exp\left(\mbx^\top \mby\right)$ as the following expectation,
\begin{equation}
    \exp(\mbx^\top \mby) = \mathbb{E}_{\omega \sim \mathcal{N}(\omega;0,\mathbf{I})}\left[\xi(\mbx,\omega)^\top\xi(\mby, \omega)\right], \label{eqn:identity}
\end{equation}
where $\xi(\cdot, \cdot): \R^D \times \R^D \rightarrow \R^l$, $l\geq 1$, is the \emph{randomized mapping} transforming the input vector to a $l$-dimensional vector %\footnote{Note that $l$ can also be 1, resulting in a scalar-valued function.} 
via a randomly drawn $\omega \sim \mathcal{N}(\omega;0,\mathbf{I})$. A classical choice of randomized mapping is to let $\xi(\mbx,\omega) = \exp{\left( \frac{1}{2}\norm{\mbx}^2\right)}\left[\sin{\left(\omega^\top \mbx\right)},\cos{\left(\omega^\top \mbx\right)}\right]^\top$ \citep{random-features,peng2021rfa}. In Performer \citep{choromanski2021rethinking}, a scalar-valued positive randomized mapping $\xi(\mbx,\omega) = \exp{\left(\omega^\top \mbx - \frac{1}{2}\norm{\mbx}^2\right)}$ is used to improve the training stability.
% , which is successfully applied in the transformer architecture named Performer. 
% Such mapping is strictly positive and enjoys better training stability. 
We base our model on the latter choice; other variants are discussed further in \cref{app:sec:random_mappings}. We use the term RFA and Performer interchangeably to refer to attention models with positive randomized mappings.

To estimate the expectation in~\cref{eqn:identity}, we can draw multiple Monte Carlo samples\footnote{This sample average can also be written as $\mbphi(\mbx,\mbw)^\top \mbphi(\mby,\mbw)$ with $\mbphi(\mbx,\mbw) \coloneqq  1/\sqrt{S}[\xi(\mbx,\omega_1), \dots, \xi(\mbx, \omega_S)]^\top \in \R^{lS}$. Here $\mbphi(\cdot,\cdot)$ are conventionally referred to as \emph{random features} \citep{random-features}. We spell out individual samples as it simplifies the analysis later.} and compute the average such that $\exp(\mbx^{\top} \mby) \approx \frac{1}{S}\sum_{s=1}^S \xi(\mbx,\omega_s)^{\top}\xi(\mby, \omega_s)$. By substituting such approximation into the softmax attention, we obtain random feature attention \citep[RFA;][]{choromanski2021rethinking,peng2021rfa}:
\begin{align*}
  &\frac{\sum_{m=1}^M\exp\left(\mbq_{n}^\top \mbk_{m} \right)\mbv_{m}^{\top}}{\sum_{m'=1}^M \exp\left(\mbq_{n}^\top \mbk_{m'} \right)} \\
  &\approx
  \frac{\sum_{m=1}^M \sum_{s=1}^S\xi(\mbq_n,\omega_s)^{\top}\xi(\mbk_m, \omega_s)\mbv_{m}^{\top}}{\sum_{m'=1}^M\sum_{s=1}^S \xi(\mbq_n,\omega_s)^{\top}\xi(\mbk_{m'}, \omega_s)} \\
  &=\frac{ \sum_{s=1}^S\xi(\mbq_n,\omega_s)^{\top}\sum_{m=1}^M\xi(\mbk_m, \omega_s)\mbv_{m}^{\top}}{\sum_{s=1}^S \xi(\mbq_n,\omega_s)^{\top}\sum_{m'=1}^M\xi(\mbk_{m'}, \omega_s)}  \numberthis\label{eqn:rfa}\\ 
  &\coloneqq \mathsf{RFA}\left(\mbq_{n},\mbK,\mbV\right). 
\end{align*}
Thanks to the linearized formulation, one can first pre-compute the corresponding key-value statistics $\sum_{m=1}^M\xi(\mbk_{m},\omega_s)\mbv_{m}^{\top}$ and $\sum_{m=1}^M\xi(\mbk_{m},\omega_s)$ once, and then reuse them for each query. Consequently, it achieves linear complexity in both time and memory with respect to the sequence length.

\subsection{Self-normalized Importance Sampling}
\label{ssec:snis}
Importance sampling (IS) is a general approach to approximating expectation $\mathbb{E}_{p(\omega)}\left[f(\omega)\right]$ when it is difficult to draw samples directly from $p(\omega)$. By sampling from a tractable \emph{proposal distribution} $q(\omega)$ instead, IS forms the following estimate to correct the sampling bias,
\begin{equation*}
    \mathbb{E}_{p(\omega)}\left[f(\omega)\right] = \mathbb{E}_{q(\omega)}\left[\frac{p(\omega)}{q(\omega)}f(\omega)\right] \approx  \frac{1}{S} \sum_{s=1}^S \frac{p(\omega_s)}{q(\omega_s)} f(\omega_s),
\end{equation*}
where $p(\omega)/q(\omega)$ is often referred to as \emph{importance weights}. Given that $q(\omega)$ is positive whenever $p(\omega) \neq 0$, IS yields an unbiased estimation. However, if the target density takes the form $p(\omega) = \tilde{p}(\omega)/Z$ and its normalizing constant is difficult to compute, IS would be intractable since it requires evaluating $p(\omega)$ explicitly. Self-normalized importance sampling (SNIS), a variant of IS estimators, mitigates this issue by taking the following form \citep{mcbook},
\begin{align*}
&    \mathbb{E}_{p(\omega)}\left[f(\omega)\right] = \frac{\mathbb{E}_{q(\omega)}\left[\frac{p(\omega)}{q(\omega)}f(\omega)\right]}{\mathbb{E}_{q(\omega)}\left[\frac{p(\omega)}{q(\omega)}\right]} \\
&    \approx   \frac{\frac{1}{S} \sum_{s=1}^S\frac{1}{Z}\frac{\tilde{p}(\omega_s)}{q(\omega_s)} f(\omega_s)}{\frac{1}{S} \sum_{s=1}^S\frac{1}{Z}\frac{\tilde{p}(\omega_s)}{q(\omega_s)}} 
    = \frac{ \sum_{s=1}^S\frac{\tilde{p}(\omega_s)}{q(\omega_s)} f(\omega_s)}{ \sum_{s=1}^S\frac{\tilde{p}(\omega_s)}{q(\omega_s)}}. \numberthis\label{eqn:snis}
\end{align*}
The name \emph{self-normalized} comes from the fact that the importance weights $p(\omega)/q(\omega)$ are normalized. Albeit at the cost of introducing a bias, this method cancels out the normalizing constant $Z$ at both nominator and denominator. SNIS often works well in practice.

\section{Randomized Attention}
\label{sec:main}
In this section, we present an alternative view of RFA, revealing new insights of how RFA approximates the softmax attention. In particular, we show that RFA can be recast as a self-normalized importance sampler and its target expectation is exactly softmax attention (\S\ref{ssec:rfa-as-snis}). This reformulation allows us to construct an unbiased estimator for softmax attention. We refer this unbiased estimation as randomized attention (\S\ref{ssec:ra}).

\subsection{RFA as Self-normalized Importance Sampling}
\label{ssec:rfa-as-snis}
Note that the formulation of RFA (\cref{eqn:rfa}) and SNIS (\S\ref{ssec:snis}) both take the form as a ratio of sample averages 
drawing from a tractable distribution. This resemblance motivates us to treat RFA as an SNIS estimator and reverse-engineer the target expectation $\mathbb{E}_{p(\omega)}\left[f(\omega)\right]$ that RFA approximates. 
For this to hold, the nominator and denominator in \cref{eqn:rfa} should define a regular importance sampling estimator and a valid importance weight up to some constant $Z$ respectively.
% \footnote{A constant $Z$ is also taken into consideration since SNIS estimates involves a ratio, where any constant will be canceled out.} 
Formally, denoting $q(\omega) \coloneqq \mathcal{N}(\omega;0,\mathbf{I})$, for any $\omega_s \sim q(\omega)$ we have
\begin{equation}
\begin{cases}
    \frac{p(\omega_s)}{q(\omega_s)}f(\omega_s) \!\!\!\!\!&= \frac{1}{Z}\xi(\mbq_n,\omega_s)^\top \sum_{m=1}^M \xi(\mbk_{m}, \omega_s)\mbv_{m}^{\top},\\
    \hfill \frac{p(\omega_s)}{q(\omega_s)} &= \frac{1}{Z} \xi(\mbq_n,\omega_s)^\top \sum_{m=1}^M\xi(\mbk_{m}, \omega_s).\\
\end{cases} \label{eqn:ra_relation}
\end{equation}
% The latent distribution $p(\omega)$, constant $Z$ and function $f(\omega)$ are unknown. 
Solving this relation gives concise formulations for both $f(\omega)$ and $p(\omega)$ (see \cref{app:sec:prop1} for the proof):
\begin{proposition}\label{prop:ra}
    Let $q(\omega) = \mathcal{N}(\omega;0,\mathbf{I})$ be the proposal, $\xi(\mbx,\omega) = \exp{\left(\omega^\top \mbx - \frac{1}{2}\norm{\mbx}^2\right)}$ be the positive randomized mapping in \citet{choromanski2021rethinking} and $\mathbb{E}_{p(\omega)}\left[f(\omega)\right]$ be the unknown target expectation. Given the relation specified in \cref{eqn:ra_relation}, the distribution $p(\omega)$ is a Gaussian mixture with parametric component weights and means,
    \begin{align*}
    p(\omega) = \sum_{m=1}^M \pi_{m} \mathcal{N}(\omega; \mbq_n + \mbk_m, \mathbf{I}), \numberthis\label{eqn:ra-density}
    \end{align*}
    % \begin{align*}
    % p(\omega) &= q(\omega)\frac{\sum_{m=1}^M  \xi(\mbq_n,\omega)^\top \xi(\mbk_{m}, \omega)}{\sum_{m'=1}^M \exp(\mbq_n^\top \mbk_{m'})} \\
    % &= \sum_{m=1}^M \pi_{m} \mathcal{N}(\omega; \mbq_n + \mbk_m, \mathbf{I}), \numberthis\label{eqn:ra-density}
    % \end{align*}
    where $\pi_m = \frac{\exp\left( \mbq_n^\top\mbk_m \right)}{\sum_{m'=1}^M\exp\left( \mbq_n^\top\mbk_{m'} \right)}$ is the component weight. Besides, $f(\omega)$ is an attention-like aggregation function over value vectors, which computes the linearized similarity between queries and keys via randomized mappings,
    \begin{equation}
    f(\omega) = \frac{\xi(\mbq_n,\omega)^\top \sum_{m=1}^M \xi(\mbk_m, \omega) \mbv_{m}^{\top}}{\xi(\mbq_n,\omega)^\top \sum_{m'=1}^M \xi(\mbk_{m'}, \omega)}.\label{eqn:ra-function}
    \end{equation}
\end{proposition}
From this perspective, for each query $\mbq_n$, RFA uses $\mcN(\omega; 0,\mathbf{I})$ as the proposal to perform self-normalized importance sampling for the following expectation,\footnote{Here we add the subscript to emphasize that both $f_n(\omega)$ and $p_n(\omega)$ is specific to a particular query $\mbq_n$.}
\begin{equation*}
    \mathbb{E}_{p_n(\omega)}\!\left[f_n(\omega)\right] \!=\! \mathbb{E}_{p_n(\omega)}\!\!\left[\frac{\xi(\mbq_n,\omega)^\top \!\sum_{m=1}^M\xi(\mbk_m, \omega) \mbv_{m}^{\top}}{ \xi(\mbq_n,\omega)^\top \!\sum_{m'=1}^M\xi(\mbk_{m'}, \omega)}\right]\!.
\end{equation*}
This re-formulation offers alternative viewpoints to understand the approximation quality of RFA. It is straightforward to see that RFA is a biased (but consistent) estimator due to the self-normalization \citep{mcbook}. In addition, RFA may exhibit large bias and variance since it only uses a standard Gaussian proposal, which is far away from the underlying input-dependent mixture $p_n(\omega)$. These may explain its inferior performance and slow convergence observed in previous studies \citep{patrick2021motionformer,tay2021long}.

\subsection{Randomized Attention}
\label{ssec:ra}
The analysis above further implies that the softmax attention itself can be formulated as an expectation.
\begin{proposition}\label{prop:softmax_as_expectation}
    Let $p_n(\omega)$ and $f_n(\omega)$ be defined by \cref{eqn:ra-density} and \cref{eqn:ra-function} respectively. Then for softmax attention we have
\begin{align*}
\mathsf{SoftmaxAttn}(\mbq_n, \mbK,\mbV) = \E_{p_n(\omega)}\left[f_n(\omega)\right].\numberthis\label{eqn:softmax_as_expectation}
\end{align*}
\end{proposition}
% \begin{align*}
%   &\mathsf{SoftmaxAttn}(\mbq_n, \mbK,\mbV) = \frac{\sum_{m=1}^M\exp(\mbq_n^\top\mbk_m)\mbv_m}{\sum_{m'=1}^M \exp(\mbq_n^\top\mbk_{m'})}\\
%   &= \E_{p_n(\omega)}\left[\frac{\sum_{m=1}^M\xi(\mbq_n,\omega)^\top \xi(\mbk_{m}, \omega)\mbv_m}{\sum_{m'=1}^M \xi(\mbq_n,\omega)^\top \xi(\mbk_{m'}, \omega)}\right]. \numberthis\label{eqn:softmax_as_expectation}
% \end{align*}
The detailed proof is in \cref{app:sec:prop2}. As a result, RFA can be viewed as using importance sampling to estimate softmax attention. Alternatively, one can directly sample from $p_n(\omega)$ to construct an \emph{unbiased} estimate of the softmax attention,
\begin{align*}
    &\mathsf{SoftmaxAttn}(\mbq_n, \mbK,\mbV) \\
    &\approx \frac{1}{S}\sum_{s=1}^S \frac{\xi(\mbq_n,\omega_s)^\top \sum_{m=1}^M\xi(\mbk_m, \omega_s) \mbv_{m}^{\top}}{ \xi(\mbq_n,\omega_s)^\top \sum_{m'=1}^M\xi(\mbk_{m'}, \omega_s)} \\
    &\coloneqq \mathsf{RA}\left(\mbq_{n},\mbK,\mbV\right)
    % \numberthis\label{eqn:ra}
\end{align*}
with $\omega_1, \dots, \omega_S \sim p_n(\omega)$. To the best of our knowledge, this is the first kernel linearization estimator that approximates the whole softmax attention, instead of just exponential kernels, in an \emph{unbiased} manner. We refer to this estimator as \emph{randomized attention} (RA), since it computes attention-like aggregations but via randomized mappings. 

Intuitively, RA constructs the randomized mapping by sampling from the contextual distribution $p_n(\omega)$, which promotes $\omega$ in the vicinity of the resultant of current queries and keys. Aware of locations of query-key pairs, $\omega$ is likely to describe their similarity better than input-agnostic ones as in RFA. 
In addition, each query position $n$ in RA induces an exclusive distribution $p_n$, which makes the randomized mapping adaptive to each query. This allows the model to process query information at a finer-grained level and thus achieves higher approximation fidelity (see \S\ref{sec:experiment} for empirical validation).
Nevertheless, the use of query-specific modeling requires to draw a different set of samples for different queries. As a result, the mapped key statistics $\xi(\mbk_m, \omega)$ will be different for different queries, which prevents reusing the computation and results in $\mathcal{O}(MN)$ complexity, rendering it less applicable in approximating softmax attention in practice.
% Nevertheless, the use of query-specific modeling requires to draw a different set of samples for different queries, preventing reusing the computation. This results in $\mathcal{O}(MN)$ complexity and renders it less applicable in approximating softmax attention in practice.

This is in sharp contrast to RFA. RFA uses the same proposal $\mcN(\omega;0,\mathbf{I})$ for all queries, and thus the modeling power is greatly reduced since the standard Gaussian would capture neither contextual information nor the inherent variations among queries. The advantage of the shared proposal is that it enables efficient computation reuse of key-value statistics (\cref{eqn:rfa}), as the same randomized mapping is reused across queries. This property accounts for RFA's linear complexity.

\section{Linear Complexity Randomized Attention}
\label{sec:randomized-attention}
%  which leads to a better trade-off between approximation accuracy and efficiency.
In this section, we propose an improved estimator of softmax attention to combine both the expressiveness of RA and the efficiency of RFA.
Motivated by the difference between RA and RFA, we generalize the importance sampling formulation of RFA by adopting \emph{multiple} proposals. This strategy not only captures query information at a finer-grained level, but also allows the model to estimate softmax attention in a query-specific manner (\S\ref{ssec:lara_mis}).
% At the same time, these proposals are shared among all queries so that computation reuse in RFA is also activated. 
We further show that computation reuse in RFA can be achieved, which leads to linear complexity computation with the help of self-normalized importance sampling (\S\ref{ssec:lara_snis}).

\subsection{Importance Sampling with Multiple Proposals}
\label{ssec:lara_mis}
As discussed in \S\ref{ssec:ra}, both RA and RFA aim to estimate the expectation $\mathbb{E}_{p_n(\omega)}\left[f_n(\omega)\right]$ (\cref{eqn:softmax_as_expectation}).
The main difference between RA and RFA is that RA samples from a distinct distribution for each query, while RFA uses the same proposal distribution for all queries. To get the best of both worlds, we propose to adopt a set of $C$ ($C \ll N$) proposal distributions $\{q_c(\omega)\}_{c=1}^C$ for our estimation, each of which depends on a \emph{subset} of queries (see \cref{app:sssec:lara_proposal_form} for the detailed discussion on parameterizing these proposals).

% , which is a well-established technique in the literature of importance sampling \citep{mcbook}.
This strategy not only enables a finer-grained treatment for query information, but also allows the model to estimate softmax attention in a query-specific way, which is the key advantage of RA. To be specific, since there are several proposals available for each query, and these proposals may provide complementary information to each other, we could combine them by invoking multiple importance sampling~\citep[MIS;][]{veach1995optimally}. For each query, the MIS estimate takes the following form,\footnote{Here we assume only one sample is drawn from each proposal distribution. A more general treatment would allow arbitrary numbers of samples to be drawn from each proposal.}
\begin{equation}
\mathbb{E}_{p_n(\omega)}\left[f_n(\omega)\right] \approx \sum_{c=1}^C \alpha_{nc}(\omega_c) \frac{p_n(\omega_c)}{q_c(\omega_c)}f_n(\omega_c)\label{eqn:lara:mis}
\end{equation}
where $\omega_c \sim q_c(\omega)$ for $c=1,\dots,C$ and $\{\alpha_{nc}(\cdot)\}_{c=1}^C$ are \emph{weighting functions}. The MIS estimator is unbiased \citep{veach1995optimally} if $\sum_{c=1}^C \alpha_{nc}(\omega) = 1$ for any $\omega$ (see the proof in \cref{app:sec:proof_for_unbiasedness_of_mis}).\footnote{Strictly speaking, for the MIS estimator to be unbiased, we additionally need the weighting functions to be zero for any $\omega$ such that $p_n(\omega) = 0$, although this holds trivially in our setting.} Intuitively, MIS first computes individual importance sampling estimates with each proposal, which are averaged together according to the \emph{query-specific} weighting functions.

Ideally, the $n$-th set of weighting functions $\{\alpha_{nc}(\cdot)\}_{c=1}^C$ should specialize in processing the $n$-th query. To accomplish this goal, we expect weighting functions to be optimal (i.e., minimize the estimation variance) for the corresponding query. Optimal weighting functions takes the following form (detailed derivation can be found in \cref{app:sec:opt}),
% \footnote{The optimality is in the sense of minimizing the variance of self-normalized importance sampling estimator; see \S\ref{ssec:lara_snis}.}
\begin{align*}
    & \alpha^*_{nc}(\omega_c) =
    \frac{q_c(\omega_c)}{\sum_{c'=1}^C q_{c'}(\omega_c)} + \\
    &\quad\quad\quad q_c(\omega_c)\left(r_{nc}(\omega_c) - \sum_{c=1}^C\frac{q_c(\omega_c)}{\sum_{c'=1}^C q_{c'}(\omega_c)}r_{nc}(\omega_c)\right).
\end{align*}
Here $r_{nc}(\cdot)$ is roughly proportional to the closeness between $q_c(\cdot)$ and the query-specific optimal proposal. Intuitively, the optimal weighting function consists of two terms. The first term is query-agnostic and the second term is a query-specific correction. The correction term is defined by the difference between $r_{nc}(\cdot)$ and its average weighted by $q_c(\cdot)$; consequently, if $r_{nc}(\cdot)$ is large, the correction term will be positive, driving the weight of the $c$-th proposal to be higher and vice versa.

In most cases, it is intractable to apply optimal weighting functions, since the closed form of $r_{nc}(\cdot)$ is not available. We therefore approximate the optimal weighting functions by the following form,
\begin{equation}
    \alpha_{nc}(\omega_c) = \frac{q_c(\omega_c)}{\sum_{c'=1}^C q_{c'}(\omega_c)} + r'_{nc} - \frac{1}{C}\sum_{c=1}^Cr'_{nc},\label{eqn:lara:opt_weighting_function}
\end{equation}
where $r'_{nc}$ measures the degree of the proposal $q_c$ favoring the $n$-th query. For tractability, we implement $r'_{nc}$ as the normalized similarity between the $n$-th query and the representation of the $c$-th query subset. We also decouple the computation between proposal densities $q_c(\omega)$ and $r'_{nc}$, so that contributions from query-agnostic and query-specific terms can be independent of each other (see \cref{app:sssec:lara_weighting_functions} for more details and ablations). Note that \cref{eqn:lara:opt_weighting_function} still ensures unbiasedness (or consistency) of MIS estimation due to $\sum_{c=1}^C \alpha_{nc}(\omega) = 1$.

\subsection{Achieving Linear Time and Space Complexity}
\label{ssec:lara_snis}
According to our MIS estimator (\cref{eqn:lara:mis}), the key-value statistics under each proposal can be pre-computed once and then reused for all queries. This implies the computation reuse in RFA is achievable and so as the linear complexity.

The only problem left now is that the MIS estimator still requires explicitly evaluating the density $p_n(\omega)$ for each query (\cref{eqn:ra-density}), which exhibits quadratic complexity. This is because $p_n(\omega)$ is a Gaussian mixture with $M$ components, incurring $\mcO(NM)$ computations in total. We show that a self-normalized version of MIS allows us to further reduce the complexity to be linear. According to \cref{prop:ra} (and \cref{app:eqn:ebm_form_of_p} in \cref{app:sec:prop1}), the mixture density $p_n(\omega)$ can be equivalently expressed as
% We show that this can be circumvented by exploiting self-normalized importance sampling. First, we note that according to , 
\begin{align*}
p_n(\omega) \!=\! \frac{\mathcal{N}(\omega;0,\mathbf{I}) \xi(\mbq_n,\omega)^\top \sum_{m=1}^M \xi(\mbk_{m}, \omega)}{\sum_{m'=1}^M \exp\left(\mbq_n^\top\mbk_{m'}\right)} \!\coloneqq\! \frac{\tilde{p}_n(\omega)}{Z_p}.
\end{align*}
Our key observation is that now the numerator contains a linearized dot product of randomized mappings, which can be pre-computed and reused for all queries, while the denominator is similar to the normalizing constant in regular softmax attention and can only be computed in quadratic time. Fortunately, the denominator can be canceled out if we adopt the \emph{self-normalized} estimator (see \S\ref{ssec:snis}),
\begin{align*}    
\mathbb{E}_{p_n(\omega)}\left[f_n(\omega)\right] &\approx 
% \frac{\sum_{c=1}^C\alpha_{nc}(\omega_c)\frac{1}{Z_p}\frac{\tilde{p}_n(\omega_c)}{q_c(\omega_c)} f_n(\omega_c)}{\sum_{c=1}^C\alpha_{nc}(\omega_c)\frac{1}{Z_p}\frac{\tilde{p}_n(\omega_c)}{q_c(\omega_c)}} \\
\frac{\sum_{c=1}^C\alpha_{nc}(\omega_c)\frac{\tilde{p}_n(\omega_c)}{q_c(\omega_c)} f_n(\omega_c)}{\sum_{c=1}^C\alpha_{nc}(\omega_c)\frac{\tilde{p}_n(\omega_c)}{q_c(\omega_c)}} \\
&\coloneqq \mathsf{LARA}\left(\mbq_{n},\mbK,\mbV\right).\numberthis\label{eqn:lara}
% \frac{\sum_{s=1}^S\frac{\tilde{p}_n(\omega_s)}{q(\omega_s)} f_n(\omega_s)}{\sum_{s=1}^S\frac{\tilde{p}_n(\omega_s)}{q(\omega_s)}}.
\end{align*}
The resulting estimator is consistent and runs with linear complexity, similar to RFA. We name it linear randomized attention (LARA). See \cref{alg:lara} in \cref{app:ssec:lara} for an algorithmic sketch of LARA.

\subsection{Discussion: RFA, RA, and LARA}
LARA defines a flexible framework to bridge RFA and RA. To delineate the connection between RFA and LARA, we find LARA can be further rewritten as (see \cref{app:sec:derivation_lara} for the derivation)
\begin{align*}    
&\mathsf{LARA}\left(\mbq_{n},\mbK,\mbV\right)\\
&=\frac{\sum_{c=1}^C \alpha'_{nc}(\omega_c)  \xi(\mbq_n,\omega_c)^\top  \sum_{m=1}^M\xi(\mbk_m, \omega_c) \mbv_{m}^{\top}}{\sum_{c=1}^C \alpha'_{nc}(\omega_c)  \xi(\mbq_n,\omega_c)^\top \sum_{m=1}^M \xi(\mbk_{m}, \omega_c)},
% \numberthis\label{eqn:lara_as_weighted_rfa}
% \frac{\sum_{s=1}^S\frac{\tilde{p}_n(\omega_s)}{q(\omega_s)} f_n(\omega_s)}{\sum_{s=1}^S\frac{\tilde{p}_n(\omega_s)}{q(\omega_s)}}.
\end{align*}
where $\omega_c \sim q_c(\omega)$ for $c = 1,\dots,C$ and $\alpha'_{nc}(\omega_c) \coloneqq \alpha_{nc}(\omega_c)\mcN(\omega_c;0, \mathbf{I})/q_c(\omega_c)$. 
% Compared to RFA (\cref{eqn:rfa}), LARA adopts a similar formulation, except that LARA draws different $\omega_c$ from each proposal and then computes the weighted summation with query-specific weights $\alpha_{nc}'$. 
Comparing to the formulation of RFA (\cref{eqn:rfa}), we see that RFA is a special case of LARA if we set all proposals to $\mcN(\omega;0,\mathbf{I})$ and all $\alpha_{nc}(\cdot)$ to constant functions.
On the other hand, LARA is equivalent to RA if we remove the use of self-normalization, set $\alpha_{nc}(\omega) = \delta_{nc}$ \footnote{That is, weighting functions now become the Kronecker delta function, where $\alpha_{nc}(\omega) = 1$ if $n = c$ and $0$ otherwise.} and maintain $N$ proposals, each of which takes the same form of $p_n(\omega)$ (\cref{eqn:ra-density}). With general proposals and weighting functions, LARA approximates softmax attention in a query-specific manner as in RA while achieving linear complexity as in RFA, effectively combining the advantages of both estimators.

\section{Experiments}
\label{sec:experiment}
In this section, we conduct extensive experiments across various domains to verify the effectiveness of linear randomized attention. Firstly, we start with an experiment to assess the approximation error of different random feature based methods (\S\ref{ssec:toy_experiment}). We then perform a number of experiments on various data modalities, including image classification (\S\ref{ssec:image_classification}), video action recognition (\S\ref{ssec:video_recognition}), machine translation (\S\ref{ssec:machine_translation}), and long sequence modeling on Long Range Arena benchmark (\cref{app:ssec:additional_lra}). Additional details as well as ablation studies can be found in \cref{app:sec:experiment_details,app:sec:experiment_results}. The implementation details of RA, Performer (RFA) and LARA are provided in \cref{app:sec:detail_ra_rfa_lara}.

\subsection{Experiments on the Approximation Quality}
\label{ssec:toy_experiment}
We conduct a preliminary experiment to assess the approximation fidelity of different random feature methods (details are deferred to \cref{app:ssec:toy}). In particular, we consider vision transformers \citep[ViT;][]{dosovitskiy2021vit,touvron21adeit}, keep $\mbQ, \mbK$ and $\mbV$ the same across attention variants, and compute the Mean Squared Error (MSE) between the outputs of true softmax attention and its approximations.
We use the \imagenet validation set (see more details in \S\ref{ssec:image_classification}) as the input data and report MSE results averaged over all images. \cref{fig:approx_err} shows the results with respect to the number of random samples under different sequence lengths. We observe that RFA (Performer) soon plateaus at large approximation error and does not improve even with more samples, possibly due to low sample efficiency. On the other hand, LARA exhibits much lower MSE than Performer and the approximation error continually decreases as the number of samples increases. As for RA, it achieves the lowest MSE among these three methods. This clearly indicates that increasing the model's resolution over query positions as in LARA and RA is more effective in improving approximation quality, compared to simply increasing the sample size from the same distribution (as in Performer).
% LARA achieves a better trade-off between accuracy and efficiency. 

\begin{figure*}[tb]
% \vskip 0.1in
\centering
\begin{subfigure}[b]{0.299\textwidth} 
  \centering
  {\includegraphics[width=\textwidth]{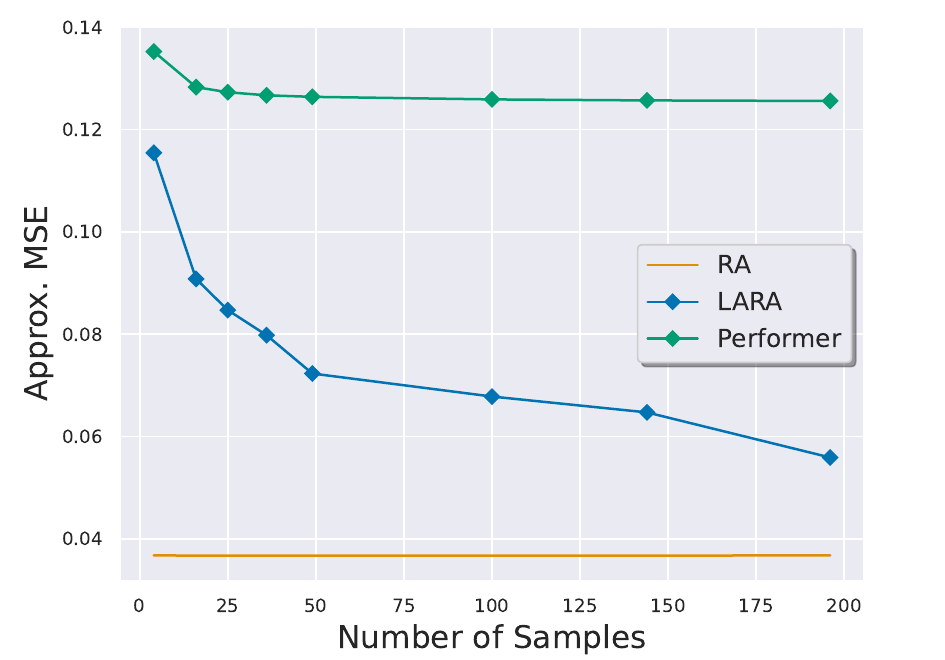}}
  \caption{Sequence length $N = 196$.}\label{fig:approx_err:196}
\end{subfigure}
\begin{subfigure}[b]{0.299\textwidth} 
  \centering
  {\includegraphics[width=\textwidth]{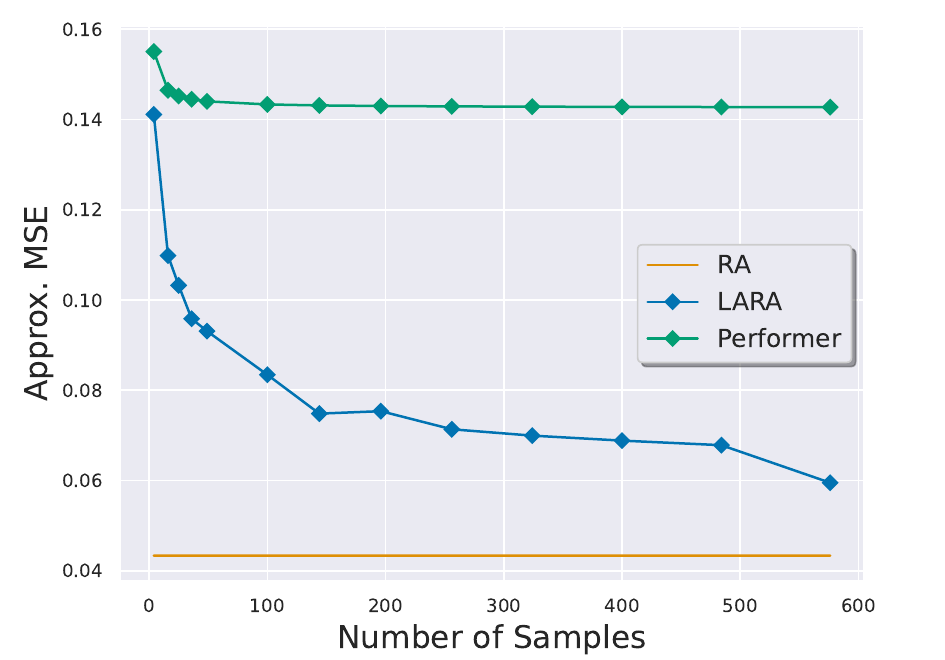}}
  \caption{Sequence length $N = 576$.}\label{fig:approx_err:576}
\end{subfigure}
\begin{subfigure}[b]{0.299\textwidth} 
  \centering
  {\includegraphics[width=\textwidth]{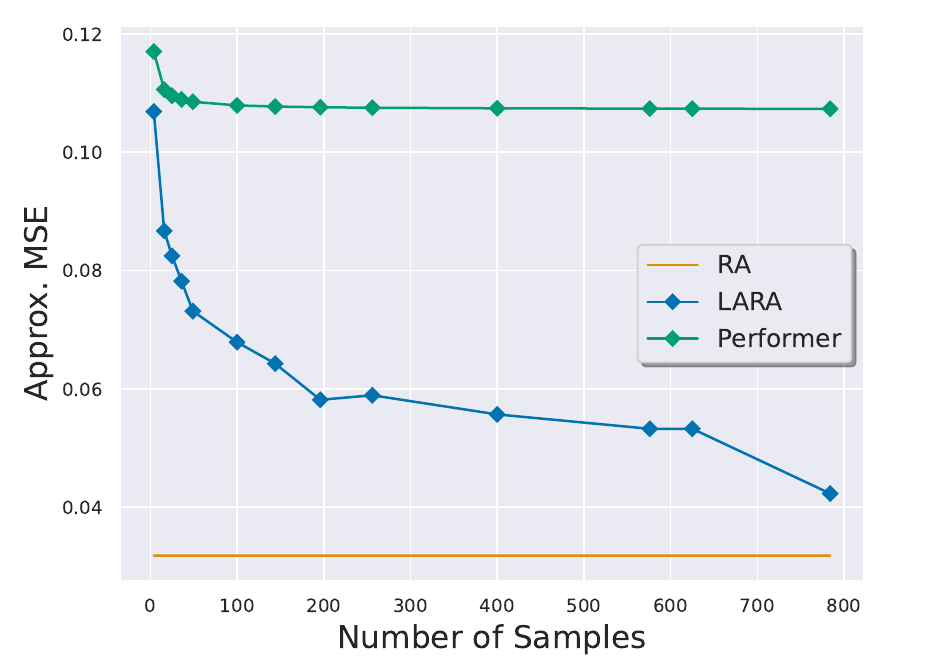}}
  \caption{Sequence length $N = 784$.}\label{fig:approx_err:784}
\end{subfigure}
% \vskip -0.1in
\caption{Mean Squared Error (MSE) between the true softmax attention and different approximation methods under different numbers of samples (lower is better). Results are evaluated on ViTs under typical settings of sequence length, including 196, 576 and 784. Note that we only draw 1 sample in RA estimation so that the curve of RA is constant.}
\label{fig:approx_err}
% \end{center}
% \vskip -0.2in
\end{figure*}

\subsection{Image Classification}
\label{ssec:image_classification}
For image classification, we conduct our experiment on the \imagenet benchmark \citep{deng2009imagenet}, which consists of approximately 1,280K/50K images over 1,000 classes for training/validation splits respectively. We apply our attention mechanism to different vision transformer (ViT) architectures \citep{dosovitskiy2021vit}, including DeiT \citep{touvron21adeit} and pyramid vision transformers v2 \citep[PVTv2;][]{pvt,pvtv2}. The former architecture adopts standard transformer layers with regular softmax attention and receives sequence with length 196 by default; while the latter processes much longer image sequences, which is therefore more suitable to evaluate the scalability of various efficient attention. More model and training details can be found in \cref{app:ssec:image}.

\begin{table}[t]
	\caption{Classification results on \imagenet dataset with DeiT architectures under different attention mechanisms. 
% 	By default, these methods are applied to DeiT model with patch size 16 (sequence length 196); 
	``*-8'' denotes the corresponding attention method with patch size 8, resulting in longer sequence with length 784; $N$ denotes the sequence length.}
	\label{tb:deit}
    \vskip 0.13in
	\centering
	\resizebox{0.95\columnwidth}{!}{    
\begin{tabular}{l | c | c  | c || c | c}
\toprule[.1em]
\multirow{2}{*}{Model} & \multirow{2}{*}{Complexity} & \multicolumn{2}{c||}{DeiT-Tiny} & \multicolumn{2}{c}{DeiT-Small} \\
% 			\cmidrule{4-7}
& & \# Param. & {Top-1 Acc.} & \# Param. & {Top-1 Acc.} \\
\midrule
Performer & $\mathcal{O}(N)$ &  5.7M & 65.92 & 22.0M & 74.29\\
Performer-8 & $\mathcal{O}(N)$ & 5.7M & 67.79 & 22.0M & 74.57\\
LARA & $\mathcal{O}(N)$  &  5.8M & 71.48 &22.2M & 79.48 \\
LARA-8 & $\mathcal{O}(N)$  &  5.8M & \textbf{74.16} &22.2M & \textbf{80.62} \\
RA & $\mathcal{O}(N^2)$  &  5.7M & 71.86 &22.0M & 80.04\\
\midrule
Softmax  & $\mathcal{O}(N^2)$ & 5.7M & 72.20 & 22.0M & 79.90\\
\bottomrule[.1em]
\end{tabular}}  
% \vskip -0.13in
\end{table}

\paragraph{Results on DeiT.} The comparison among different random feature based methods on DeiT model is demonstrated in \cref{tb:deit}. Consistent with previous studies \citep{ripple}, Performer (RFA) incurs a significant performance drop due to its limited modeling capacity. Its unbiased counterpart, RA, performs much better than Performer and even slightly outperforms exact softmax attention under larger model sizes. This empirically validates the expressiveness of unbiasedness in approximating softmax attention. LARA achieves a good trade-off between Performer and RA. It enjoys linear complexity as Performer but performs substantially better.
On the other hand, we note that a linear complexity variant enables the transformer model to scale to much longer sequences, which is often prohibitive for traditional softmax attention but delivers better predictive performance \citep{el2021xcit}. We thus train Performer and LARA with $8 \times 8$ image patches (resulting in sequence length 784) with all other settings unchanged. As shown in \cref{tb:deit}, increasing the sequence length (suffixed with ``-8'') consistently boosts model performance. However, LARA benefits from longer sequences much more significantly than Performer and outperforms softmax attention by a large margin. This indicates the potential modeling power of our framework for long sequences. Also see \cref{app:ssec:additional_imagenet} for additional experiments and ablations.

\begin{table}[t]
	\caption{Classification results on \imagenet dataset compared with state-of-the-art model architectures.}
	\label{tb:pvt-sota}
    \vskip 0.13in
	\centering
	\resizebox{0.95\columnwidth}{!}{    
		\begin{tabular}{l | c | c | c}
			\toprule[.1em]
			Model & \# Param. & FLOPs & {Top-1 Acc.}\\
    		\midrule
    		PVT-v1-T \citep{pvt} & 13.2M & 2.1G & 75.1 \\
    		SOFT-T \citep{lu2021soft} & 13.1M &  1.9G & 79.3 \\
    		RegionViT-T \citep{regionvit} & 13.8M & 2.4G & \textbf{80.4}\\
    		\midrule
    		PVT-v2-b1 (SRA) & 14.0M & 2.1G & 78.7\\
    		PVT-v2-b1 + Performer & 12.1M & 2.5G & 77.3\\
    		PVT-v2-b1 + LARA & 13.7M & 2.3G & 79.6\\
    		\midrule
    		\midrule
            PVT-v1-S \citep{pvt} &24.5M & 3.8G & 79.8\\
            DeiT-S \citep{touvron21adeit} &22.1M & 4.6G & 79.9\\
            RegNetY-4G \citep{regnet} &21.0M  &4.0G & 80.0\\
            Swin-T \citep{liu2021swin} &28.3M & 4.5G & 81.3\\
            CvT-13 \citep{cvt} &20.0M & 4.5G & 81.6\\
            Twins-SVT-S \citep{twins} & 24.0M & 2.8G & 81.7\\
            SOFT-S \citep{lu2021soft} & 24.1M & 3.3G & 82.2 \\
            Focal-T \citep{focal} & 29.1M & 4.9G & 82.2 \\
            ViL-S \citep{vil} & 24.6M & 4.9G & 82.4 \\
    		\midrule
    		PVT-v2-b2 (SRA) & 25.4M & 4.0G & 82.1\\
    		PVT-v2-b2 + Performer & 21.1M & 4.9G & 81.0\\
    		PVT-v2-b2 + LARA & 22.4M & 4.5G & \textbf{82.6}\\
    		\midrule
    		\midrule
            PVTv1-M \citep{pvt} & 44.2M & 6.7G & 81.2\\
            RegNetY-8G \citep{regnet} & 39.0M & 8.0G & 81.7\\
            CvT-21 \citep{cvt} & 32.0M & 7.1G & 82.5\\
            SOFT-M \citep{lu2021soft} & 45.0M & 7.2G & 82.9\\
            RegionViT-M \citep{regionvit} & 42.0M & 7.9G & 83.4\\
            ViL-M \citep{vil} & 39.7M & 9.1G & 83.5\\
            \midrule
            PVT-v2-b3 (SRA) & 45.2M & 6.9G & 83.3\\
            PVT-v2-b3 + Performer & 36.0M & 8.2G & 82.4\\
            PVT-v2-b3 + LARA & 39.9M & 7.7G & \textbf{83.6}\\
            \midrule
            \midrule
            PVTv1-L \citep{pvt} &  61.4M &  9.8G &  81.7\\
            RegNetY-16G \citep{regnet} & 84.0M & 16.0G & 82.9 \\
            Swin-S \citep{liu2021swin} & 50.0M & 8.7G & 83.0\\
            SOFT-L \citep{lu2021soft} & 64.1M & 11.0G & 83.1\\
            Focal-S \citep{focal} & 51.1M & 9.1G & 83.5\\
            ViL-B \citep{vil} & 55.7M & 13.4G & 83.7\\
            RegionViT-B \citep{regionvit} & 73.8M & 13.6G & 83.8\\
    		\midrule
    		PVT-v2-b4 (SRA) & 62.6M & 10.1G & 83.6\\
    		PVT-v2-b4 + Performer & 48.6M & 11.9G & 82.7\\
    		PVT-v2-b4 + LARA & 54.5M & 11.3G & \textbf{84.0}\\
			\bottomrule[.1em]
	\end{tabular}}  
% 	\vskip -0.15in
\end{table}

\paragraph{Results on PVTv2.} We then apply our method to the strong baseline PVTv2 and compare it against recent state-of-the-art model architectures. As presented by \cref{tb:pvt-sota}, we observe although replacing spatial reduction attention (SRA; details in \S\ref{app:ssec:image}) with Performer leads to inferior performance, LARA brings a consistent performance gain over vanilla SRA with much fewer model parameters. In addition, PVTv2 with LARA even performs highly competitive with state-of-the-art architectures across various model sizes, without introducing other inductive biases (such as locality). This implies the superior modeling capacity of LARA compared to SRA and Performer.

\subsection{Video Action Recognition}
\label{ssec:video_recognition}
In this section, we test our method on video action recognition with video transformers. We consider two standard datasets: (1) Kinetics-400 \citep[K400;][]{kay2017kinetics}, which contains 238,574 videos for training and 19,877 for evaluation at the time of writing and (2) Something-something-v2 \citep[SSv2;][]{goyal2017ssv2}, consisting of around 168K/25K videos of 174 classes for training/validation splits respectively. We base our model on the Motionformer architecture \citep{patrick2021motionformer} and follow their training and evaluation protocol; more details can be found in \cref{app:ssec:video}.

\cref{tb:video-diff-attn} reports the top-1 classification accuracy for both \kinetics and \ssv datasets. We see that RA still achieves the best performance among attention approximations albeit falling behind the exact softmax attention.
Since Motionformer is pretrained on images with \emph{softmax attention}, this gap is likely introduced by employing a different attention mechanism during training the model further on video datasets.
% Since Motionformer is initialized by weights pretrained on images with \emph{softmax attention}, this gap is likely introduced by the discrepancy between different attention mechanisms at pre-training and fine-tuning phases.
Besides, LARA outperforms Performer and Nystr\"omformer \citep{xiong2021nystromformer} by a large margin on both \kinetics and \ssv datasets.
Although achieving strong performance, Orthoformer \citep{patrick2021motionformer} runs much slower (roughly $3\times$ or more) than other attention variants due to its sequential nature. As a result, LARA achieves better trade-offs than these baselines between predictive accuracy and efficiency.
\begin{table}[t]
	\caption{Video action recognition accuracy on \kinetics and \ssv datasets with different attention mechanisms. $N$ denotes the spatial sequence length.}
	\label{tb:video-diff-attn}
    \vskip 0.13in
	\centering
	\resizebox{\columnwidth}{!}{    
		\begin{tabular}{l | c | c | c }
			\toprule[.1em]
			Model & Complexity & {Acc. (\%) on \kinetics} & {Acc. (\%) on \ssv}\\
    		\midrule
    		Nystr\"omformer & $\mathcal{O}(N)$ & 76.5 & 61.7\\
    		Orthoformer & $\mathcal{O}(N)$ & 77.8 & 64.7 \\
    		Performer & $\mathcal{O}(N)$ & 72.1 & 53.1\\
    		LARA & $\mathcal{O}(N)$ & 77.5 & 63.7\\
    		RA & $\mathcal{O}(N^2)$ & 78.2 & 64.9  \\
    		\midrule
    		Exact Motionformer & $\mathcal{O}(N^2)$ & 79.2 & 66.5\\
			\bottomrule[.1em]
	\end{tabular}}  
% 	\vskip -0.1in
\end{table}

\subsection{Machine Translation}
\label{ssec:machine_translation}
In this section, we conduct experiments on WMT14 EN--DE machine translation benchmark \citep{bojar2014wmt} to evaluate the performance of our model under various sequence lengths. We follow \citet{vaswani2017attention} and \citet{ott2018scalingNMT} to
preprocess this dataset, resulting in about 4.5M/3K/3K sentences pairs for training/validation/testing splits respectively. We adopt the standard transformer base architecture \citep{vaswani2017attention} and replace encoder self-attention with efficient attention variants.
More detailed configurations are deferred to \cref{app:ssec:mt}.

\begin{table}[t]
	\caption{Test BLEU scores on WMT14 EN-DE dataset under different attention mechanisms. For brevity, ``\# samples'' denotes either the number of samples or landmarks involved in different attention variants; -- indicates the model does not converge during training; n.a. denotes not applicable.}
	\label{tb:wmt14-en-de}
	\vskip 0.13in
	\centering
	\resizebox{0.7\columnwidth}{!}{    
		\begin{tabular}{l | c | c | c }
\toprule[.1em]
Model & \# samples & \# Param. & {BLEU}\\
\midrule
% Softmax attention & n.a. & 60.92M & 27.3\\
Softmax & n.a. & 60.92M & 27.5\\
\midrule
\multirow{3}{*}{ABC}
& 16 & 60.93M & 25.4 \\
& 32 & 60.94M & 25.6 \\
& 64 & 60.95M & 26.0 \\
\midrule
\multirow{3}{*}{Linformer}
& 16 & 60.92M & 17.4 \\
& 32 & 61.31M & 23.0 \\
& 64 & 61.70M & 23.7 \\
\midrule
\multirow{3}{*}{Nystr\"omformer}
& 16 & 60.92M & 25.1\\
& 32 & 60.92M & 26.8\\
& 64 & 60.92M & 26.8\\
\midrule
\multirow{4}{*}{Performer} 
& 64 & 60.92M & --\\
& 128 & 60.92M & 23.5\\
& 256 & 60.92M & 23.7 \\
& 512 & 60.92M & 23.3\\
\midrule
\multirow{3}{*}{LARA}
& 16 & 60.96M & 26.4 \\
& 32 & 60.96M & 26.8 \\
& 64 & 60.96M & 27.0 \\
\midrule
{RA} & n.a. & 60.92M & \textbf{27.8} \\
\bottomrule[.1em]
	\end{tabular}}  
% 	\vskip -0.1in
\end{table}

\cref{tb:wmt14-en-de} presents the test BLEU scores under different attention mechanisms. Since this dataset consists mostly of short sentences, we set the number of samples to be relatively smaller. However, the training of Performer is quite unstable and a larger number of samples is required to mitigate this issue. Besides, we observe a similar trend that replacing the standard softmax attention with Performer leads to a significant performance drop, while increasing the number of samples does not improve the translation quality. RA, on the other hand, even outperforms softmax attention by over 0.3 BLEU score, clearly demonstrating the modeling capacity of unbiased approximations. LARA reaches performance close to softmax attention while runs with the same complexity as Performer; compared to other attention variants, LARA outperforms both Linformer \citep{wang2020linformer} and ABC \citep{peng2021abc} while obtaining similar BLEU scores to Nystr\"omformer \citep{xiong2021nystromformer}. This indicates RA and LARA are also capable of modeling natural language, which is typically hierarchically structured.

% We did not report the results of randomized attention since its memory consumption is larger than softmax attention, which runs out of memory for longer sequences.

% \section{Empirical Analysis}
% \label{sec:experimental_analysis}
\subsection{Analysis on Time and Memory Consumption}
\label{ssec:running_time_and_memory}
To evaluate the empirical efficiency of various attention methods, we conduct a simulation on a standard transformer architecture and report the running time and memory consumption under different sequence lengths. The detailed setup can be found in \cref{app:ssec:time_mem}. As shown in \cref{fig:time_mem} (and \cref{app:tb:time_mem} in \cref{app:ssec:time_mem} for exact statistics), we note that RA runs twice (or more) as slow as ordinary softmax attention with about $2.5\times$ memory consumption. This is as expected since RA needs to first compute full softmax probabilities to sample from $p_n$, and then compute $f_n$, both of which take a similar amount of computation to softmax attention. Nevertheless, its efficient variant LARA runs as fast as Performer with marginally increased memory usage. As for another baseline Nystr\"omformer \citep{xiong2021nystromformer}, which we found is a strong baseline and is used across experiments, it runs much slower than other variants at relatively short sequence lengths (e.g., less than 8192). Overall, the comparison result validates that LARA achieves a good balance between efficiency and expressiveness.
% both DEIT and RIPPLE (Naïve) come with quadratic complexity in the number of tokens. We observe
% that RIPPLE with dynamic programming (DP) performs significantly better than RIPPLE (Naïve),
% which demonstrates the effectiveness of our dynamic programming algorithm. Furthermore, RIPPLE
% behaves similarly to DEIT-LA as the number of tokens increases, verifying that it could be executed in
% linear observed time. When processing a large number of tokens, RIPPLE often achieves a 5× or even
% 10× reduction in running time and memory compared to its quadratic counterparts.
\begin{figure}[tb]
% \vskip 0.05in
\centering
\begin{subfigure}[b]{0.2383\textwidth} 
  \centering
  {\includegraphics[width=\textwidth]{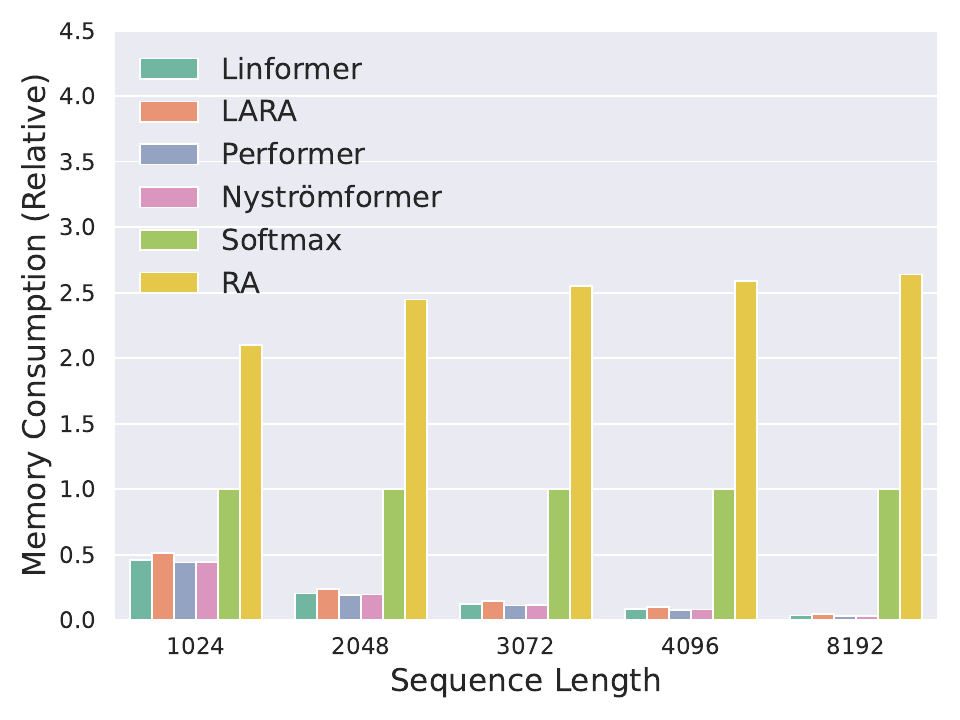}}
  \caption{Memory consumption.}\label{fig:mem}
\end{subfigure}
\begin{subfigure}[b]{0.2383\textwidth} 
  \centering
  {\includegraphics[width=\textwidth]{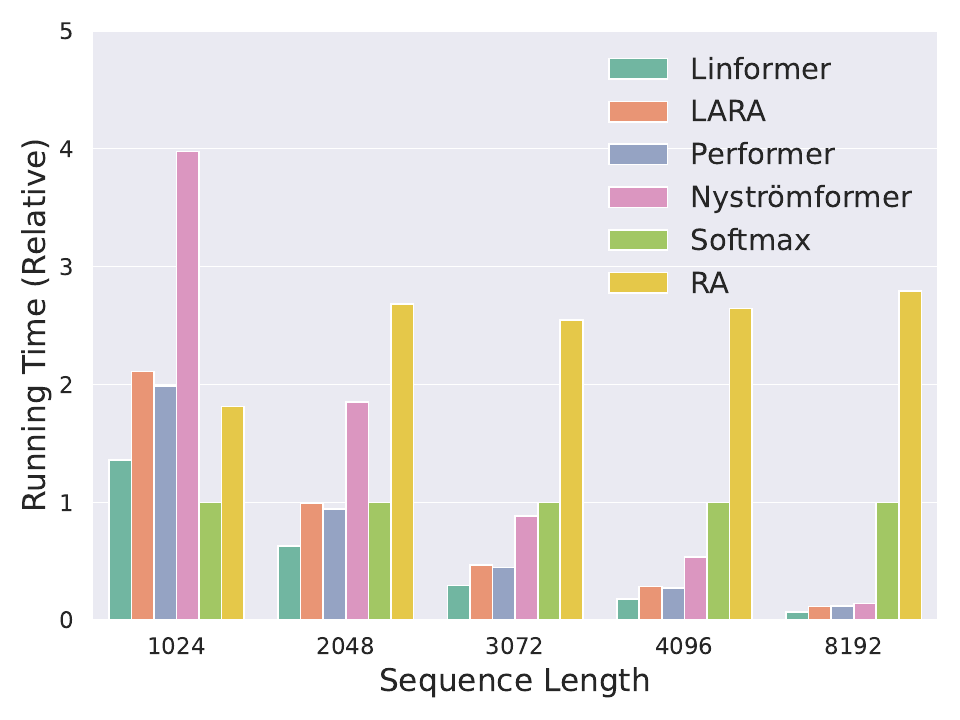}}
  \caption{Running time.}\label{fig:time}
\end{subfigure}
% \vskip -0.05in
\caption{Empirical memory consumption (left) and running time (right) of different attention mechanisms under different sequence lengths. Metrics are measured relative to the softmax attention.}
\label{fig:time_mem}
% \end{center}
% \vskip -0.15in
\end{figure}

\section{Related Work}
\label{sec:related_work}
%%%%%%%%%%%%%%%%%%%%%%%%%%%%%%%%%%%%%%%%%%%%%%%%%%%%%%%%%%%%%%%%%%%
Transformer models \citep{vaswani2017attention} are difficult to scale to long sequences due to the quadratic time and space complexity of self-attention mechanisms. Recently, a significantly large number of approaches have been proposed to improve the efficiency of attention mechanisms. A widely adopted paradigm is to utilize sparse attention, where each query is limited to only attend a subset of tokens. Such sparse attentive patterns can be pre-defined, such as sliding windows \citep{beltagy2020longformer} or block-wise local chunks \citep{liu2018generating,parmar2018image,child2019generating,ainslie2020etc,zaheer2020bigbird, liu2021swin}; alternatively, the model can adaptively select tokens to take into account. This can be done via a trainable top-$k$ selecting operator \citep{pietruszka2020sparsifying}, learnable hash functions \citep{Kitaev2020reformer, daras2020smyrf}, clustering with K-Means \citep{vyas2020fast,roy2021routing} or grouping tokens with a differentiable sorting module \citep{yi2020sinkhorn}. 
% On the other hand, Informer \citep{zhou2021informer} implements sparse computations by directly selecting a subset of queries that produces the most diversified attention maps. 
More recently, Combiner \citep{ren2021combiner} is proposed to apply the sparse mechanism to factorize the softmax probability distribution so that the resulting approximation runs with sub-quadratic time but achieves full attention capacity.  

Low-rank approximations to the softmax attention also received considerable interest. For instance, the Nystr\"om method can be adopted to approximate the softmax attention map by a sub-sampled matrix \citep{xiong2021nystromformer}. Another approach is the kernel linearization, which aims to decompose the exponential kernel into a dot product of feature maps. Such feature maps can be randomized that yield unbiased estimates of exponential kernels \citep{choromanski2021rethinking,peng2021rfa}, or deterministic that enjoy better training convergence \citep{katharopoulos2020transformers_are_rnns, kasai2021t2r, schlag2021linear}. Alternatively, one can use a learnable matrix (including Linformer \citep{wang2020linformer} and ABC \citep{peng2021abc}) or other downsampling operations \citep{dai2020funnel,pvt,pvtv2} to project the key-value pairs into fixed-length sequences. Besides, a set of auxiliary points can also be incorporated to cache the information from the long sequence via an attention mechanism, which is adopted in LUNA \citep{ma2021luna}, Set transformer \citep{lee2019set} and Perceiver \citep{jaegle2021perceiverio,jaegle2021perceiver}. 
Our work falls into the category of kernel linearization methods, but in contrast to previous works, we propose an unbiased estimation for the \emph{whole} softmax attention, which has not been explored and is orthogonal to previous works.

Recent studies also consider combining both the sparse and low-rank bias to achieve better approximation \citep{nguyen2021fmmformer,zhu2021long-short,chen2021scatterbrain}, or replace the softmax attention with other token-mixing mechanisms \citep{lee2021fnet,lu2021soft,chen2021skyformer,tay2021synthesizer}.
We refer readers to \citet{tay2020efficient,tay2021long,lin2021survey} for a more detailed review on advances in the topic of efficient attention.

\section{Conclusion}
In this paper, we revisit the recently proposed random feature methods for approximating the softmax attention. By recasting RFA as self-normalized importance samplers, we identify an elusive bias in its approximation process. Built on this finding, we propose the unbiased estimation, called randomized attention (RA), which constructs positive random features via query-specific distributions. We then develop a novel linear complexity self-attention mechanism called linear randomized attention (LARA), which combines the expressiveness in RA and the efficiency in RFA. Extensive experiments demonstrate the effectiveness of RA and LARA, across various domains.

\section*{Acknowledgements}
We thank Jianbo Yuan, Xiang Gao, Xiujun Li, Yanghua Peng, Ding Zhou, and Ruofan Ding for helpful discussions and feedback on early drafts of this paper. This research was supported in part by the joint research scheme of the National Natural Science Foundation of China (NSFC) and the Research Grants Council (RGC) under grant number N\_HKU714/21.

\bibliography{refs}
\bibliographystyle{icml2022}

%%%%%%%%%%%%%%%%%%%%%%%%%%%%%%%%%%%%%%%%%%%%%%%%%%%%%%%%%%%%%%%%%%%%%%%%%%%%%%%
%%%%%%%%%%%%%%%%%%%%%%%%%%%%%%%%%%%%%%%%%%%%%%%%%%%%%%%%%%%%%%%%%%%%%%%%%%%%%%%
% APPENDIX
%%%%%%%%%%%%%%%%%%%%%%%%%%%%%%%%%%%%%%%%%%%%%%%%%%%%%%%%%%%%%%%%%%%%%%%%%%%%%%%
%%%%%%%%%%%%%%%%%%%%%%%%%%%%%%%%%%%%%%%%%%%%%%%%%%%%%%%%%%%%%%%%%%%%%%%%%%%%%%%
\newpage
\appendix
\onecolumn
\textbf{\huge Appendices}\vspace{0.05in}

\section{Proof for \cref{prop:ra}}
\label{app:sec:prop1}
Assume $q(\omega) = \mathcal{N}(\omega;0,\mathbf{I})$. Recall that we define
\begin{align}
    \frac{p(\omega)}{q(\omega)}f(\omega)&= \frac{1}{Z}\sum_{m=1}^M \xi(\mbq_n,\omega)^\top \xi(\mbk_{m}, \omega)\mbv_{m}^{\top}, \label{app:eqn:ra_relation_numerator}\\
    \hfill \frac{p(\omega)}{q(\omega)} &= \frac{1}{Z}\sum_{m=1}^M \xi(\mbq_n,\omega)^\top \xi(\mbk_{m}, \omega). \label{app:eqn:ra_relation_denominator}
\end{align}
\paragraph{On solving $f(\omega)$.}
Substituting the second equality into the first one yields the form of $f(\omega)$:
\begin{equation}
    f(\omega) = \frac{\sum_{m=1}^M \xi(\mbq_n,\omega)^\top \xi(\mbk_m, \omega) \mbv_{m}^{\top}}{\sum_{m'=1}^M \xi(\mbq_n,\omega)^\top \xi(\mbk_{m'}, \omega)} = \frac{\xi(\mbq_n,\omega)^\top \sum_{m=1}^M \xi(\mbk_m, \omega) \mbv_{m}^{\top}}{\xi(\mbq_n,\omega)^\top \sum_{m'=1}^M \xi(\mbk_{m'}, \omega)}.
\end{equation}
To more clearly illustrate the connection between RFA and RA, one can also manually verify this by rearranging terms in \cref{eqn:rfa}:
\begin{align*}
    & \mathsf{RFA}\left(\mbq_{n},\mbK,\mbV\right) \\
    &= \frac{\sum_{s=1}^S \sum_{m=1}^M \xi(\mbq_n,\omega_s)^\top \xi(\mbk_m, \omega_s) \mbv_{m}^{\top}}{\sum_{s=1}^S \sum_{m'=1}^M \xi(\mbq_n,\omega_s)^\top \xi(\mbk_{m'}, \omega_s)}\\
    &= \frac{\sum_{s=1}^S \left(\textcolor{keywords}{\sum_{m'=1}^M \xi(\mbq_n,\omega_s)^\top \xi(\mbk_{m'}, \omega_s)}\right) \frac{\sum_{m=1}^M \xi(\mbq_n,\omega_s)^\top \xi(\mbk_m, \omega_s) \mbv_{m}^{\top}}{\textcolor{keywords}{\sum_{m'=1}^M \xi(\mbq_n,\omega_s)^\top \xi(\mbk_{m'}, \omega_s)}}}{\sum_{s=1}^S \sum_{m'=1}^M \xi(\mbq_n,\omega_s)^\top \xi(\mbk_{m'}, \omega_s)} \\
    &= \frac{\sum_{s=1}^S \frac{Zp(\omega_s)}{q(\omega_s)} \textcolor{keywords}{\frac{\sum_{m=1}^M \xi(\mbq_n,\omega_s)^\top \xi(\mbk_m, \omega_s) \mbv_{m}^{\top}}{\sum_{m'=1}^M \xi(\mbq_n,\omega_s)^\top \xi(\mbk_{m'}, \omega_s)}}}{\sum_{s=1}^S \frac{Zp(\omega_s)}{q(\omega_s)}} \coloneqq \frac{\sum_{s=1}^S \frac{Zp(\omega_s)}{q(\omega_s)} \textcolor{keywords}{f(\omega_s)}}{\sum_{s=1}^S \frac{Zp(\omega_s)}{q(\omega_s)}}.
\end{align*}

\paragraph{On solving $p(\omega)$.}
According to \cref{app:eqn:ra_relation_denominator}, we have
\begin{align*}
    p(\omega) = \frac{q(\omega) \left[\sum_{m=1}^M \xi(\mbq_n,\omega)^\top \xi(\mbk_{m}, \omega)\right]}{Z},
\end{align*}
where $Z$ is the partition function. Recall that in \cref{eqn:identity}
\begin{equation}
    \mathbb{E}_{\omega \sim\mathcal{N}(\omega;0,\mathbf{I})}\left[\xi(\mbx,\omega)^\top\xi(\mby, \omega)\right] = \int \xi(\mbx,\omega)^\top\xi(\mby, \omega)q(\omega)d\omega= \exp(\mbx^\top \mby),\label{app:eqn:rand_map_int_1}
\end{equation}
which further implies
\begin{align*}
    Z &= \int q(\omega) \left[\sum_{m=1}^M \xi(\mbq_n,\omega)^\top \xi(\mbk_{m}, \omega)\right] d\omega = \sum_{m=1}^M \int \xi(\mbq_n,\omega)^\top \xi(\mbk_{m}, \omega) q(\omega) d\omega = \sum_{m=1}^M \exp(\mbq_n^\top \mbk_{m}).
\end{align*}
Therefore,
\begin{align*}
    p(\omega) &= q(\omega)\frac{\sum_{m=1}^M  \xi(\mbq_n,\omega)^\top \xi(\mbk_{m}, \omega)}{\sum_{m'=1}^M \exp(\mbq_n^\top \mbk_{m'})} \numberthis\label{app:eqn:ebm_form_of_p}\\
    &= \sum_{m=1}^M\frac{\exp(\mbq_n^\top \mbk_{m})}{\sum_{m'=1}^M \exp(\mbq_n^\top \mbk_{m'})}\frac{  q(\omega) \xi(\mbq_n,\omega)^\top \xi(\mbk_{m}, \omega)}{\exp(\mbq_n^\top \mbk_{m})} \\
    &= \sum_{m=1}^M p(m)p(\omega|m), \numberthis\label{app:eqn:p_as_mixture}
\end{align*}
which is effectively a mixture distribution and each component is selected with probability proportional to the similarity of queries and keys. As long as the randomized mapping is non-negative, $p(\omega|m)$ would be a valid probability distribution since its density would be non-negative and integrate to 1, according to \cref{app:eqn:rand_map_int_1}.

In terms of the particular form of the distribution, we have the following lemma:
\begin{lemma}\label{app:lemma:pos_map_induce_gaussian}
Assume $\xi(\mbx,\omega) = \exp\left(\omega^\top \mbx - \frac{\norm{\mbx}^2}{2}\right)$ and $q(\omega) = \mathcal{N}(\omega;0,\mathbf{I})$. Given two vectors $\mbq_n$ and $\mbk_m$ with the same dimension as $\omega \in \R^D$, if a density function $g(\omega)$ w.r.t. the random vector $\omega$ is defined as
\begin{equation*}
    g(\omega) \coloneqq \frac{q(\omega) \xi(\mbq_n,\omega)^\top \xi(\mbk_{m}, \omega)}{\exp(\mbq_n^\top \mbk_{m})},
\end{equation*}
Then $\omega \sim \mathcal{N}(\omega;\mbq_n + \mbk_m, \mathbf{I})$.
\end{lemma}
\begin{proof}
Note that $q(\omega) = \mathcal{N}(\omega;0,\mathbf{I}) = \frac{1}{(2\pi)^{d/2}}\exp\left(-\frac{1}{2}\omega^\top \omega\right)$. Based on the ``complete the square'' technique, we have
\begin{align*}
    g(\omega) 
    &= \frac{q(\omega) \xi(\mbq_n,\omega)^\top \xi(\mbk_{m}, \omega)}{\exp(\mbq_n^\top \mbk_{m})} \\
    &= \frac{\exp\left(-\frac{1}{2}\omega^\top \omega\right) \exp\left(\omega^\top \mbq_n- \frac{\norm{\mbq_n}^2}{2}\right) \exp\left(\omega^\top \mbk_m- \frac{\norm{\mbk_m}^2}{2}\right)}{(2\pi)^{d/2}\exp(\mbq_n^\top \mbk_{m})} \\
    &=\frac{\exp\left(-\frac{1}{2}\omega^\top\omega + \omega^\top \left(\mbq_n+\mbk_m\right)\right)\exp\left(-\frac{1}{2}\norm{\mbk_m}^2-\frac{1}{2}\norm{\mbq_n}^2 \right)}{(2\pi)^{d/2}\exp(\mbq_n^\top \mbk_{m})}  \\
    &=\frac{\exp\left(-\frac{1}{2}\omega^\top\omega + \omega^\top \left(\mbq_n+\mbk_m\right) - \frac{1}{2}\left(\mbq_n+\mbk_m\right)^\top \left(\mbq_n+\mbk_m\right)\right)\exp\left( \mbq_n^\top\mbk_m \right)}{(2\pi)^{d/2}\exp(\mbq_n^\top \mbk_{m})}  \\
    &= \frac{1}{(2\pi)^{d/2}} \exp\left(-\frac{1}{2}\left[\omega - (\mbq_n+\mbk_m)\right]^\top \left[\omega - (\mbq_n+\mbk_m)\right]\right),
\end{align*}
which is exactly the density function of a multivariate Gaussian with the mean $\mbq_n + \mbk_m$ and covariance $\mathbf{I}$. 
\end{proof}
Following \cref{app:lemma:pos_map_induce_gaussian}, it is straightforward to obtain  
\begin{align*}
    p(\omega) &= q(\omega)\frac{\sum_{m=1}^M  \xi(\mbq_n,\omega)^\top \xi(\mbk_{m}, \omega)}{\sum_{m'=1}^M \exp(\mbq_n^\top \mbk_{m'})}\\
    &= \sum_{m=1}^M \frac{q(\omega) \xi(\mbq_n,\omega)^\top \xi(\mbk_{m}, \omega)}{\sum_{m'=1}^M \exp(\mbq_n^\top \mbk_{m'})}\\
    &= \sum_{m=1}^M \frac{\exp(\mbq_n^\top \mbk_{m})}{\sum_{m'=1}^M \exp(\mbq_n^\top \mbk_{m'})}\frac{q(\omega) \xi(\mbq_n,\omega)^\top \xi(\mbk_{m}, \omega)}{\exp(\mbq_n^\top \mbk_{m})}\\
    &= \sum_{m=1}^M \frac{\exp(\mbq_n^\top \mbk_{m})}{\sum_{m'=1}^M \exp(\mbq_n^\top \mbk_{m'})} \mathcal{N}(\omega; \mbq_n + \mbk_m, \mathbf{I})\\
    &\coloneqq \sum_{m=1}^M \pi_m \mathcal{N}(\omega; \mbq_n + \mbk_m, \mathbf{I}).
\end{align*}
where $\pi_m = \frac{\exp\left( \mbq_n^\top\mbk_m \right)}{\sum_{m'=1}^M\exp\left( \mbq_n^\top\mbk_{m'} \right)}$.

\paragraph{Discussion.} Due to the dependence on the randomized mapping $\xi(\cdot,\cdot)$, different choices of feature maps would yield distinct density forms. Here we mainly study the positive randomized mapping in Performer \citep{choromanski2021rethinking} and leave other choices (such as trigonometric functions in \citet{peng2021rfa}) as future work. 

\section{Proof for \cref{prop:softmax_as_expectation}}
\label{app:sec:prop2}
Since the vanilla random-feature-based attention estimation is consistent, softmax attention must be equal to expected randomized attention. However, such equality can also be verified as follows.
Assume $q(\omega) = \mcN(\omega;0,\mathbf{I})$ and $p(\omega) = q(\omega)\frac{\sum_{m=1}^M  \xi(\mbq_n,\omega)^\top \xi(\mbk_{m}, \omega)}{\sum_{m'=1}^M \exp(\mbq_n^\top \mbk_{m'})}$ given by Proposition 1. Then we have 
\begin{equation}
    \frac{q(\omega)}{p(\omega)} = \frac{\sum_{m'=1}^M \exp(\mbk_{m'}^\top\mbq_n)}{\sum_{m=1}^M\xi(\mbq_n,\omega)^\top \xi(\mbk_{m}, \omega)}. \label{app:eqn:helper1}
\end{equation}
In addition, according to the definition of randomized mappings $\xi(\cdot,\cdot)$,
\begin{equation}
    \mathbb{E}_{\omega \sim\mathcal{N}(\omega;0,\mathbf{I})}\left[\xi(\mbx,\omega)^\top\xi(\mby, \omega)\right] = \int \xi(\mbx,\omega)^\top\xi(\mby, \omega)q(\omega)d\omega= \exp(\mbx^\top \mby).\label{app:eqn:helper2}
\end{equation}

Equipped with these helpers, we are ready to derive the equality as follows:
\begin{align*}
  &\E\left[\mathsf{RA}(\mbq_n, \mbK, \mbV)\right]\\
  &= \E_{p(\omega)}\left[\frac{\sum_{m=1}^M\xi(\mbq_n,\omega)^\top \xi(\mbk_{m}, \omega)\mbv_{m}^{\top}}{\sum_{m'=1}^M \xi(\mbq_n,\omega)^\top \xi(\mbk_{m'}, \omega)}\right]\\
  &= \E_{p(\omega)}\left[\frac{\textcolor{keywords}{\sum_{m'=1}^M \exp(\mbk_{m'}^\top\mbq_n)}}{\sum_{m=1}^M\xi(\mbq_n,\omega)^\top \xi(\mbk_{m}, \omega)}\frac{\sum_{m=1}^M\xi(\mbq_n,\omega)^\top \xi(\mbk_{m}, \omega)\mbv_{m}^{\top}}{\textcolor{keywords}{\sum_{m'=1}^M \exp(\mbk_{m'}^\top\mbq_n)}}\right]\\
  &= \E_{p(\omega)}\left[\frac{q(\omega)}{p(\omega)}\frac{\sum_{m=1}^M\xi(\mbq_n,\omega)^\top \xi(\mbk_{m}, \omega)\mbv_{m}^{\top}}{\sum_{m'=1}^M \exp(\mbk_{m'}^\top\mbq_n)}\right] && \rhd \text{\cref{app:eqn:helper1}}  \\
   &= \E_{\textcolor{keywords}{q(\omega)}}\left[\frac{\sum_{m=1}^M\xi(\mbq_n,\omega)^\top \xi(\mbk_{m}, \omega)\mbv_{m}^{\top}}{\sum_{m'=1}^M \exp(\mbk_{m'}^\top\mbq_n)}\right] \\
   &= \frac{\sum_{m=1}^M \E_{q(\omega)} \left[\xi(\mbq_n,\omega)^\top \xi(\mbk_{m}, \omega)\right]\mbv_{m}^{\top}}{\sum_{m'=1}^M \exp(\mbk_{m'}^\top\mbq_n)} && \rhd \text{linearity of expectations}\\   
    &= \frac{\sum_{m=1}^M\exp(\mbk_m^\top\mbq_n)\mbv_{m}^{\top}}{\sum_{m'=1}^M \exp(\mbk_{m'}^\top\mbq_n)} && \rhd \text{\cref{app:eqn:helper2}}\\
    &= \mathsf{SoftmaxAttn}(\mbq_n, \mbK,\mbV) 
\end{align*}

\section{Discussion on Different Randomized Mappings}
\label{app:sec:random_mappings}
The randomized mapping $\xi(\cdot, \cdot)$ transforms the inputs to a $l$-dimensional vector. There are various choices of $\xi(\cdot, \cdot)$ for the resulting estimator to become unbiased in the context of attention mechanisms, such as
\begin{itemize}
    \item $l=1$ and $\xi(\mbx,\omega) = \exp{\left(\omega^\top \mbx- \frac{\norm{\mbx}^2}{2}\right)}$ in \citet{choromanski2021rethinking};
    \item $l=1$ and $\xi(\mbx,\omega) = \sqrt{2}\exp{\left( \frac{\norm{\mbx}^2}{2}\right)}\cos{\left(\omega^\top \mbx + b \right)}$ with $b \sim \operatorname{Uniform}(0, 2\pi)$ in \citet{random-features};
    \item $l=2$ and $\xi(\mbx, \omega) = \left[\exp{\left( \frac{\norm{\mbx}^2}{2}\right)}\sin{\left(\omega^\top \mbx\right)},\exp{\left( \frac{\norm{\mbx}^2}{2}\right)}\cos{\left(\omega^\top \mbx\right)}\right]$ in \citet{random-features, peng2021rfa};
    \item $l=2$ and $\xi(\mbx, \omega) = \left[\frac{1}{\sqrt{2}}\exp{\left(\omega^\top \mbx - \frac{\norm{\mbx}^2}{2}\right)},\frac{1}{\sqrt{2}}\exp{\left(-\omega^\top \mbx - \frac{\norm{\mbx}^2}{2}\right)}\right]$ in \citet{choromanski2021rethinking}.
\end{itemize}
In the main paper, we focus on the positive randomized mappings \citep{choromanski2021rethinking}; for other positive randomized mappings, it is also possible to derive a similar target expectation, such as the hyperbolic randomized mapping proposed in \citet{choromanski2021rethinking}:
\begin{corollary}\label{corollary:ra_density_hyperbolic}
Consider the hyperbolic randomized mapping
\begin{equation*}
    \xi(\mbx, \omega) \!=\! \frac{1}{\sqrt{2}}\exp\!\left(-\frac{\norm{\mbx}^2}{2}\right)\!\left[\exp\!{\left(\omega^\top \mbx\right)},\exp\!{\left(-\omega^\top \mbx \right)}\right]^{\!\top}\!.
\end{equation*}
It also implies an SNIS estimator of $\mathbb{E}_{p(\omega)}\left[f(\omega)\right]$, where the function $f(\omega)$ remains the same as \cref{eqn:ra-function} and the density $p(\omega)$ is also a Gaussian mixture as follows:
\begin{equation*}
    \frac{1}{2} \sum_{m=1}^M \pi_{m} \left(\mathcal{N}(\omega; \mbq_n + \mbk_m, \mathbf{I}) + \mathcal{N}(\omega; - \mbq_n - \mbk_m, \mathbf{I})\right).
\end{equation*}
\end{corollary}
\begin{proof}
Consider the hyperbolic positive randomized mapping
\begin{equation*}
    \xi(\mbx, \omega) = \frac{1}{\sqrt{2}}\exp\left(-\frac{\norm{\mbx}^2}{2}\right)\left[\exp{\left(\omega^\top \mbx\right)},\exp{\left(-\omega^\top \mbx \right)}\right]^\top.
\end{equation*}
According to proof of \cref{prop:ra} in \cref{app:sec:prop1},  the density function $p(\omega)$ corresponding to the hyperbolic randomized mapping should also be a mixture (\cref{app:eqn:p_as_mixture}) with the following form
\begin{align*}
    p(\omega) &= \sum_{m=1}^M\frac{\exp(\mbq_n^\top \mbk_{m})}{\sum_{m'=1}^M \exp(\mbq_n^\top \mbk_{m'})}\frac{  q(\omega) \xi(\mbq_n,\omega)^\top \xi(\mbk_{m}, \omega)}{\exp(\mbq_n^\top \mbk_{m})} \coloneqq \sum_{m=1}^M \pi_m p(\omega|m),
\end{align*}
where $\pi_m \coloneqq \frac{\exp(\mbq_n^\top \mbk_{m})}{\sum_{m'=1}^M \exp(\mbq_n^\top \mbk_{m'})}$ and $p(\omega|m)$ denotes the density of the $m$-th component distribution. By substituting the hyperbolic randomized mapping into the equation above, we have
\begin{align*}
    &p(\omega|m) \\
    &= \frac{  q(\omega) \xi(\mbq_n,\omega)^\top \xi(\mbk_{m}, \omega)}{\exp(\mbq_n^\top \mbk_{m})}\\
    &= \frac{1}{2}\frac{  q(\omega) \left[\exp\left(\omega^\top \mbq_n- \frac{\norm{\mbq_n}^2}{2}\right) \exp\left(\omega^\top \mbk_m- \frac{\norm{\mbk_m}^2}{2}\right) + \exp\left(-\omega^\top \mbq_n- \frac{\norm{\mbq_n}^2}{2}\right) \exp\left(-\omega^\top \mbk_m- \frac{\norm{\mbk_m}^2}{2}\right)\right]}{\exp(\mbq_n^\top \mbk_{m})}\\
    &= \frac{1}{2}\frac{q(\omega) \exp\left(\omega^\top \mbq_n- \frac{\norm{\mbq_n}^2}{2}\right) \exp\left(\omega^\top \mbk_m- \frac{\norm{\mbk_m}^2}{2}\right)}{\exp(\mbq_n^\top \mbk_{m})} + \frac{1}{2}\frac{q(\omega)\exp\left(-\omega^\top \mbq_n- \frac{\norm{\mbq_n}^2}{2}\right) \exp\left(-\omega^\top \mbk_m- \frac{\norm{\mbk_m}^2}{2}\right)}{\exp(\mbq_n^\top \mbk_{m})}\\
\end{align*}
It is straightforward to recognize that this can be viewed as the sum of two densities. We then invoke \cref{app:lemma:pos_map_induce_gaussian} for each of them, which results in two Gaussians
\begin{align*}
    p(\omega|m) = \frac{1}{2}\mathcal{N}(\omega;\mbq_n + \mbk_m, \mathbf{I}) + \frac{1}{2}\mathcal{N}(\omega;-\mbq_n - \mbk_m, \mathbf{I}).
\end{align*}
Therefore, the true density function $p(\omega)$ can be expressed as follows
\begin{align*}
    p(\omega) &= \pi_m p(\omega|m) = \frac{1}{2} \sum_{m=1}^M \pi_{m} \left(\mathcal{N}(\omega; \mbq_n + \mbk_m, \mathbf{I}) + \mathcal{N}(\omega; - \mbq_n - \mbk_m, \mathbf{I})\right).
\end{align*}
% \begin{align*}
%     p(\omega) &= \sum_{m=1}^M\frac{\exp(\mbq_n^\top \mbk_{m})}{\sum_{m'=1}^M \exp(\mbq_n^\top \mbk_{m'})}\frac{  q(\omega) \xi(\mbq_n,\omega)^\top \xi(\mbk_{m}, \omega)}{\exp(\mbq_n^\top \mbk_{m})} \\
%     &= \sum_{m=1}^M\frac{\exp(\mbq_n^\top \mbk_{m})}{\sum_{m'=1}^M \exp(\mbq_n^\top \mbk_{m'})}\frac{  q(\omega) \exp\left(\omega^\top \mbq_n- \frac{\norm{\mbq_n}^2}{2}\right) \exp\left(\omega^\top \mbk_m- \frac{\norm{\mbk_m}^2}{2}\right) + \exp\left(-\omega^\top \mbq_n- \frac{\norm{\mbq_n}^2}{2}\right) \exp\left(-\omega^\top \mbk_m- \frac{\norm{\mbk_m}^2}{2}\right)}{\exp(\mbq_n^\top \mbk_{m})}
% \end{align*}
\end{proof}
However, it is much more difficult to analyze the classical random Fourier mappings \citep{random-features} since they may involve a negative density. As a result, the formulation of RFA with these randomized mappings may not define a valid self-normalized importance sampling estimate. We study positive randomized mappings through this paper and leave investigation into other cases as future work.

\section{Analysis on the Optimal Weighting Function in Multiple Importance Sampling}
\label{app:sec:opt}
In this section, we analyze the optimal weighting function in MIS, which is self-normalized in our setting (\S\ref{ssec:lara_snis}).

Given the set of $N$ queries $\mbQ$ and the set of $M$ key-value pairs $\mbK$ and $\mbV$, the regular softmax attention can be expressed as expected randomized attention according to \cref{eqn:softmax_as_expectation}:
\begin{equation*}
  \frac{\sum_{m=1}^M\exp(\mbq_n^\top\mbk_m)\mbv_{m}^{\top}}{\sum_{m'=1}^M \exp(\mbq_n^\top\mbk_{m'})} =  \E_{p_n(\omega)}\left[\frac{\sum_{m=1}^M\xi(\mbq_n,\omega)^\top \xi(\mbk_{m}, \omega)\mbv_{m}^{\top}}{\sum_{m'=1}^M \xi(\mbq_n,\omega)^\top \xi(\mbk_{m'}, \omega)}\right] \coloneqq \E_{p_n(\omega)}\left[f_n(\omega)\right] = \mbmu_n,
\end{equation*}
where the distribution is defined in \cref{prop:ra} as
\begin{equation*}
    p_n(\omega) = \mcN(\omega;0,\mathbf{I})\frac{\sum_{m=1}^M  \xi(\mbq_n,\omega)^\top \xi(\mbk_{m}, \omega)}{\sum_{m'=1}^M \exp(\mbq_n^\top \mbk_{m'})} = \sum_{m=1}^M \frac{\exp(\mbq_n^\top\mbk_m)}{\sum_{m'=1}^M \exp(\mbq_n^\top\mbk_{m'})} \mathcal{N}(\omega; \mbq_n + \mbk_m, \mathbf{I}).
\end{equation*}
The attention mechanism outputs a $D$-dimensional vector for each query. For brevity, we start with considering the $d$-th dimension and denote $f_{n,d}(\omega)$ as the $d$-th dimension of the function output at query position $n$. We then have
\begin{equation*}
  \E_{p_n(\omega)}\left[f_{n,d}(\omega)\right] =
  \E_{p_n(\omega)}\left[\frac{\sum_{m=1}^M\xi(\mbq_n,\omega)^\top \xi(\mbk_{m}, \omega)v_{m,d}}{\sum_{m'=1}^M \xi(\mbq_n,\omega)^\top \xi(\mbk_{m'}, \omega)}\right] = \frac{\sum_{m=1}^M\exp(\mbq_n^\top\mbk_m)v_{m,d}}{\sum_{m'=1}^M \exp(\mbq_n^\top\mbk_{m'})} \coloneqq \mu_{n,d}.
\end{equation*}
In our work, we estimate the expectation above by self-normalized multiple importance sampling (see \S\ref{ssec:lara_snis}). For the $d$-th dimension of the output at query position $n$, we have
\begin{align*}
    \hat{g}_{n,d} \coloneqq \frac{\sum_{c=1}^C \E_{q_c(\omega)}\left[\alpha_{nc}(\omega)\frac{p_n(\omega)}{q_c(\omega)}f_{n,d}(\omega)\right]}{\sum_{c=1}^C\E_{q_c(\omega)}\left[\alpha_{nc}(\omega)\frac{p_n(\omega)}{q_c(\omega)}\right]} \approx \frac{\sum_{c=1}^C\alpha_{nc}(\omega_c)\frac{p_n(\omega_c)}{q_c(\omega_c)} f_{n,d}(\omega_c)}{\sum_{c=1}^C\alpha_{nc}(\omega_c)\frac{p_n(\omega_c)}{q_c(\omega_c)}} \coloneqq \frac{A}{B},
\end{align*}
where $\omega_c \sim p_c(\omega)$ for $c = 1,\dots,C$. We also let $A$ and $B$ represent the nominator and denominator respectively. The expectations of $A$ and $B$ are
\begin{align*}
    \mu_A &\coloneqq \sum_{c=1}^C \E_{q_c(\omega)}\left[\alpha_{nc}(\omega)\frac{p_n(\omega)}{q_c(\omega)}f_{n,d}(\omega)\right] = \E_{p_n(\omega)}\left[f_{n,d}(\omega)\right] = \mu_{n,d};\\
    \mu_B &\coloneqq \sum_{c=1}^C\E_{q_c(\omega)}\left[\alpha_{nc}(\omega)\frac{p_n(\omega)}{q_c(\omega)}\right] = \E_{p_n(\omega)}\left[1\right] = 1.
\end{align*}
Unfortunately, the exact form of the variance of $\hat{g}_{n,d}$ is mostly intractable to compute. To this end, we follow previous practices \citep{mcbook} and approximate $\Var\left[\hat{g}_{n,d}\right]$ via the delta method. In particular, we apply the first-order Taylor expansion approximation to the function $g(A,B) \coloneqq A/B$ around point $\left(\mu_A, \mu_B\right)$, yielding
\begin{align*}
    \frac{A}{B} = g(A,B)
    &\approx g(\mu_A, \mu_B) +  
    \frac{\partial g(A,B)}{\partial A}\Bigr|_{\substack{A=\mu_A\\B=\mu_B}}(A - \mu_A) +  
    \frac{\partial g(A,B)}{\partial B}\Bigr|_{\substack{A=\mu_A\\B=\mu_B}}(B - \mu_B)\\
    &\coloneqq  g(\mu_A, \mu_B) + g_A (A - \mu_A) 
    + g_B(B - \mu_B),
    \numberthis\label{app:eqn:taylor_original}
\end{align*}
where we denote $g_A \coloneqq \frac{\partial g(A,B)}{\partial A}\Bigr|_{\substack{A=\mu_A\\B=\mu_B}}$ and $g_B \coloneqq \frac{\partial g(A,B)}{\partial B}\Bigr|_{\substack{A=\mu_A\\B=\mu_B}}$ similarly. Note that both $g_A$ and $g_B$ are constants with respect to $\omega$.
According to \cref{app:eqn:taylor_original}, the approximate expectation is the following
\begin{align*}
    \mathbb{E}\left[g(A,B)\right]
    &\approx \mathbb{E}\left[g(\mu_A, \mu_B)\right] + \mathbb{E}\left[g_A (A - \mu_A)  + g_B (B - \mu_B)\right] \\
    &=g(\mu_A, \mu_B) + g_A\mathbb{E}\left[(A - \mu_A)\right] 
     + g_B\mathbb{E}\left[(B - \mu_B)\right] \\
    &= g(\mu_A, \mu_B)
\end{align*}
and its second moment is given by 
% \begin{align*}
%     g(A,B)
%     &\approx g(\mu_A, \mu_B) + (A - \mu_A) 
%     \frac{\partial g(A,B)}{\partial A}\Bigr|_{\substack{A=\mu_A\\B=\mu_B}} + (B - \mu_B) 
%     \frac{\partial g(A,B)}{\partial B}\Bigr|_{\substack{A=\mu_A\\B=\mu_B}}\\
%     &= g(\mu_A, \mu_B) + \frac{1}{\mu_B}(A - \mu_A) - \frac{\mu_A}{\mu_B^2} (B - \mu_B)\\
%     &= \mu + (A - \mu) - \mu (B - 1) = A - \mu B + \mu \numberthis\label{app:eqn:taylor_expanded}
% \end{align*}
% and consequently $\mathbb{E}\left[g(A,B)\right] = g(\mu_A, \mu_B) = \mu$.
\begin{align*}
    &\mathbb{E}\left[g(A,B)^2\right]\\ 
    &= \mathbb{E}\left[ g(\mu_A,\mu_B)^2 + g_A^2 (A - \mu_A)^2 + g_B^2 (B - \mu_B)^2 + 2g_Ag_B(A - \mu_A)(B - \mu_B)\right]\\
    &= g(\mu_A,\mu_B)^2 + g_A^2 \Var\left[A\right] + g_B^2 \Var\left[B\right] + 2g_Ag_B\Cov\left(A, B\right).
\end{align*}
It is straightforward to compute that
\begin{align*}
\Var\left[A\right] &= \sum_{c=1}^C \Var_{q_c(\omega)}\left[\alpha_{nc}(\omega)\frac{p_n(\omega)}{q_c(\omega)}f_{n,d}(\omega)\right] \\
&= \sum_{c=1}^C \E_{q_c(\omega)}\left[\alpha_{nc}^2(\omega)\frac{p_n^2(\omega)}{q_c^2(\omega)}f_{n,d}^2(\omega)\right] - \E_{q_c(\omega)}\left[\alpha_{nc}(\omega)\frac{p_n(\omega)}{q_c(\omega)}f_{n,d}(\omega)\right]^2 \\
\Var\left[B\right] &= \sum_{c=1}^C \Var_{q_c(\omega)}\left[\alpha_{nc}(\omega)\frac{p_n(\omega)}{q_c(\omega)}\right] = \sum_{c=1}^C \E_{q_c(\omega)}\left[\alpha_{nc}^2(\omega)\frac{p_n^2(\omega)}{q_c^2(\omega)}\right] - \E_{q_c(\omega)}\left[\alpha_{nc}(\omega)\frac{p_n(\omega)}{q_c(\omega)}\right]^2 \\
\Cov\left(A, B\right) &= \sum_{c=1}^C \sum_{c'=1}^C \Cov\left(\alpha_{nc}(\omega_c)\frac{p_n(\omega_c)}{q_c(\omega_c)}f_{n,d}(\omega_c), \alpha_{nc'}(\omega_{c'})\frac{p_n(\omega_{c'})}{q_{c'}(\omega_{c'})}\right) \\
&= \sum_{c=1}^C \E_{q_c(\omega)}\left[\alpha_{nc}^2(\omega)\frac{p_n^2(\omega)}{q_c^2(\omega)}f_{n,d}(\omega)\right] - \E_{q_c(\omega)}\left[\alpha_{nc}(\omega)\frac{p_n(\omega)}{q_c(\omega)}f_{n,d}(\omega)\right]\E_{q_c(\omega)}\left[\alpha_{nc}(\omega)\frac{p_n(\omega)}{q_c(\omega)}\right].\\
g_A^2 &= \frac{1}{\mu_B^2} = 1, \\
g_B^2 &= \frac{\mu_A^2}{\mu_B^4} = \mu_{n,d}^2,\\
g_Ag_B &= -\frac{\mu_A}{\mu_B^3} = -\mu_{n,d}.
\end{align*}
The first three lines hold since $\omega_c$ is independent of $\omega_{c'}$ for any $c \neq c'$.
Therefore, the approximate variance of our estimate at the $d$-th dimension can be written as
\begin{align*}
    &\Var\left[\hat{g}_{n,d}\right] \\
    &= \Var\left[g(A,B)\right] \\
    &= \mathbb{E}\left[g(A,B)^2\right] - \mathbb{E}\left[g(A,B)\right]^2 \\
    &\approx g_A^2 \Var\left[A\right] + g_B^2 \Var\left[B\right] + 2g_Ag_B\Cov\left(A, B\right)\\
    &= \sum_{c=1}^C  \left(\E_{q_c(\omega)}\left[\alpha_{nc}^2(\omega)\frac{p_n^2(\omega)}{q_c^2(\omega)}f_{n,d}^2(\omega)\right] - \E_{q_c(\omega)}\left[\alpha_{nc}(\omega)\frac{p_n(\omega)}{q_c(\omega)}f_{n,d}(\omega)\right]^2\right) + \\
    &\quad\quad\quad \mu_{n,d}^2 \left(\E_{q_c(\omega)}\left[\alpha_{nc}^2(\omega)\frac{p_n^2(\omega)}{q_c^2(\omega)}\right] - \E_{q_c(\omega)}\left[\alpha_{nc}(\omega)\frac{p_n(\omega)}{q_c(\omega)}\right]^2\right) - \\
    &\quad\quad\quad 2\mu_{n,d} \left(\sum_{c=1}^C \E_{q_c(\omega)}\left[\alpha_{nc}^2(\omega)\frac{p_n^2(\omega)}{q_c^2(\omega)}f_{n,d}(\omega)\right] - \E_{q_c(\omega)}\left[\alpha_{nc}(\omega)\frac{p_n(\omega)}{q_c(\omega)}f_{n,d}(\omega)\right]\E_{q_c(\omega)}\left[\alpha_{nc}(\omega)\frac{p_n(\omega)}{q_c(\omega)}\right]\right)\\
    &= \sum_{c=1}^C  \E_{q_c(\omega)}\left[\alpha_{nc}^2(\omega)\frac{p_n^2(\omega)}{q_c^2(\omega)}\left(f_{n,d}^2(\omega) - 2f_{n,d}(\omega)\mu_{n,d} + \mu_{n,d}^2 \right)\right] - \E_{q_c(\omega)}\left[\alpha_{nc}(\omega)\frac{p_n(\omega)}{q_c(\omega)}f_{n,d}(\omega)\right]^2 -\\
    &\quad\quad\quad \E_{q_c(\omega)}\left[\alpha_{nc}(\omega)\frac{p_n(\omega)}{q_c(\omega)}\mu_{n,d}\right]^2 + 2\E_{q_c(\omega)}\left[\alpha_{nc}(\omega)\frac{p_n(\omega)}{q_c(\omega)}f_{n,d}(\omega)\right]\E_{q_c(\omega)}\left[\alpha_{nc}(\omega)\frac{p_n(\omega)}{q_c(\omega)}\mu_{n,d}\right]\\
    &= \sum_{c=1}^C  \E_{q_c(\omega)}\left[\alpha_{nc}^2(\omega)\frac{p_n^2(\omega)}{q_c^2(\omega)}\left(f_{n,d}(\omega) - \mu_{n,d}\right)^2\right] - \left(\E_{q_c(\omega)}\left[\alpha_{nc}(\omega)\frac{p_n(\omega)}{q_c(\omega)}f_{n,d}(\omega)\right] - \E_{q_c(\omega)}\left[\alpha_{nc}(\omega)\frac{p_n(\omega)}{q_c(\omega)}\mu_{n,d}\right]\right)^2\\
    &= \sum_{c=1}^C  \E_{q_c(\omega)}\left[\alpha_{nc}^2(\omega)\frac{p_n^2(\omega)}{q_c^2(\omega)}\left(f_{n,d}(\omega) - \mu_{n,d}\right)^2\right] - \E_{q_c(\omega)}\left[\alpha_{nc}(\omega)\frac{p_n(\omega)}{q_c(\omega)}\left(f_{n,d}(\omega) - \mu_{n,d}\right)\right]^2.
\end{align*}

Since we are using the same proposal distribution to estimate the output for all dimensions, we are interested in the sum of variance over every dimension (i.e., the trace of the covariance matrix):
\begin{align*}
    \sum_{d=1}^D \Var\left[\hat{g}_{n,d}\right] &\approx \sum_{d=1}^D\sum_{c=1}^C  \E_{q_c(\omega)}\left[\alpha_{nc}^2(\omega)\frac{p_n^2(\omega)}{q_c^2(\omega)}\left(f_{n,d}(\omega) - \mu_{n,d}\right)^2\right] - \E_{q_c(\omega)}\left[\alpha_{nc}(\omega)\frac{p_n(\omega)}{q_c(\omega)}\left(f_{n,d}(\omega) - \mu_{n,d}\right)\right]^2 \\
    &= \sum_{c=1}^C  \E_{q_c(\omega)}\left[\alpha_{nc}^2(\omega)\frac{p_n^2(\omega)}{q_c^2(\omega)}\sum_{d=1}^D\left(f_{n,d}(\omega) - \mu_{n,d}\right)^2\right] - \sum_{d=1}^D\E_{q_c(\omega)}\left[\alpha_{nc}(\omega)\frac{p_n(\omega)}{q_c(\omega)}\left(f_{n,d}(\omega) - \mu_{n,d}\right)\right]^2\\
    &= \sum_{c=1}^C  \E_{q_c(\omega)}\left[\alpha_{nc}^2(\omega)\frac{p_n^2(\omega)}{q_c^2(\omega)}\norm{f_n(\omega) - \mbmu_n}^2\right] - \sum_{d=1}^D\E_{q_c(\omega)}\left[\alpha_{nc}(\omega)\frac{p_n(\omega)}{q_c(\omega)}\left(f_{n,d}(\omega) - \mu_{n,d}\right)\right]^2. \numberthis\label{app:eqn:sum_var_over_d}
\end{align*}

According to our design choice, $\alpha_{nc}(\cdot)$ is specific to each query position. Ideally, we hope these weighting functions can minimize the sum of variance at each position. Formally, we have
\begin{align*}
\numberthis \label{app:eqn:optimal_objective}
\underset{\{\alpha_{nc}\}_{c=1}^C}{\text{minimize}} & \; \; \;  \sum_{d=1}^D \Var\left[\hat{g}_{n,d}\right] \\
\text{subject to}  & \; \; \;  \sum_{c=1}^C \alpha_{nc}(\omega) = 1  \text{ for any } \omega.
\end{align*}
Optimizing weighting functions to minimize the variance of ordinary MIS estimator has been studied by a recent work \citep{optimal_mis}. Our setting is different from it in that (1) we focus on self-normalized MIS and that (2) the function $f(\cdot)$ is vector-valued instead of scalar-valued. These differences lead to a distinct objective (\cref{app:eqn:sum_var_over_d}). Here we adapt the analysis \citep{optimal_mis} to solve this problem. In particular,
% There are other constraints such as that the density function should never be negative and that $q(\omega) > 0$ should hold whenever $p_n(\omega)f_n(\omega) \neq 0$. To simplify derivations, we leave out these constraints and verify them for the obtained density. 
we rely on the calculus of variations and introduce the following Lagrangian
\begin{equation}
    \mathcal{L}(\alpha, \lambda) =  \sum_{d=1}^D \Var\left[\hat{g}_{n,d}\right] - \int \lambda \left(\sum_{c=1}^C \alpha_{nc}(\omega) - 1\right) d\omega.
\end{equation}
Solving $\frac{\partial \mathcal{L}(\alpha, \lambda)}{\partial \alpha_{nc}} = 0$ and $\frac{\partial \mathcal{L}(\alpha, \lambda)}{\partial \lambda} = 0$ respectively yields
\begin{align}
    &2\alpha_{nc}(\omega)\frac{p_n^2(\omega)}{q_c(\omega)}\norm{f_n(\omega) - \mbmu_n}^2 - 2\sum_{d=1}^D p_n(\omega)\left(f_{n,d}(\omega) - \mu_{n,d}\right)r_{ncd} - \lambda = 0;\label{app:eqn:lagrangian_alpha}\\
    &\sum_{c=1}^C\alpha_{nc}(\omega) = 1.\label{app:eqn:lagrangian_lambda}
\end{align}
Here we denote $r_{ncd} \coloneqq \int\alpha_{nc}(\omega)p_n(\omega)\left(f_{n,d}(\omega) - \mu_{n,d}\right)d\omega$. We then rearrange \cref{app:eqn:lagrangian_alpha} to obtain
\begin{equation}
    \alpha_{nc}(\omega) = \frac{q_c(\omega)}{2p_n^2(\omega)\norm{f_n(\omega) - \mbmu_n}^2}\lambda + q_c(\omega)\frac{\sum_{d=1}^D p_n(\omega)\left(f_{n,d}(\omega) - \mu_{n,d}\right)r_{ncd}}{p_n^2(\omega)\norm{f_n(\omega) - \mbmu_n}^2}. \label{app:eqn:alpha_expression_temp1}
\end{equation}
Substituting \cref{app:eqn:alpha_expression_temp1} into \cref{app:eqn:lagrangian_lambda} gives
\begin{equation}
    \lambda = \frac{2p_n^2(\omega)\norm{f_n(\omega) - \mbmu_n}^2}{\sum_{c=1}^C q_c(\omega)}-2\frac{\sum_{c=1}^C q_c(\omega)\sum_{d=1}^D p_n(\omega)\left(f_{n,d}(\omega) - \mu_{n,d}\right)r_{ncd}}{\sum_{c'=1}^C q_{c'}(\omega)}. \label{app:eqn:lambda_expression}
\end{equation}
Substituting \cref{app:eqn:lambda_expression} back into \cref{app:eqn:alpha_expression_temp1} yields
\begin{align*}
&\alpha_{nc}(\omega) \\
&= \frac{q_c(\omega)}{\sum_{c=1}^C q_c(\omega)}\left(1 - \frac{\sum_{c=1}^C q_c(\omega)\sum_{d=1}^D p_n(\omega)\left(f_{n,d}(\omega) - \mu_{n,d}\right)r_{ncd}}{p_n^2(\omega)\norm{f_n(\omega) - \mbmu_n}^2}\right) + q_c(\omega)\frac{\sum_{d=1}^D p_n(\omega)\left(f_{n,d}(\omega) - \mu_{n,d}\right)r_{ncd}}{p_n^2(\omega)\norm{f_n(\omega) - \mbmu_n}^2}\\
&= \frac{q_c(\omega)}{\sum_{c=1}^C q_c(\omega)}\left(1 - \sum_{c=1}^C q_c(\omega)r_{nc}(\omega)\right) + q_c(\omega)r_{nc}(\omega)\\
&= \frac{q_c(\omega)}{\sum_{c'=1}^C q_{c'}(\omega)} + q_c(\omega)\left(r_{nc}(\omega) - \sum_{c=1}^C\frac{q_c(\omega)}{\sum_{c'=1}^C q_{c'}(\omega)}r_{nc}(\omega)\right)\numberthis\label{app:eqn:alpha_expression}
\end{align*}
where we denote
\begin{equation*}
    r_{nc}(\omega) \coloneqq \frac{\sum_{d=1}^D p_n(\omega)\left(f_{n,d}(\omega) - \mu_{n,d}\right)r_{ncd}}{p_n^2(\omega)\norm{f_n(\omega) - \mbmu_n}^2}.
\end{equation*}
The characteristic of existence and uniqueness of the optimal weighting function is similar to \citet{optimal_mis}.
Intuitively, the optimal weighting functions can be obtained by first calculating a query-dependent correction term, which sums to 0, and then adding such correction to the original balance heuristic weighting function. For large $r_{nc}$, the correction term will be positive, driving the weights for the $c$-th proposal to be higher; and vice versa. Such formulation introduces the dependence between the current proposal index $c$ and the target query position $n$, which allows the weighting functions $\alpha_{nc}$ (and thus the estimator) to specialize in the current query.

To obtain the exact form of $r_{nc}(\omega)$, we need to solve $r_{ncd} = \int\alpha_{nc}(\omega)p_n(\omega)\left(f_{n,d}(\omega) - \mu_{n,d}\right)d\omega$. However, deriving a closed form solution is mostly intractable given its complex structure, which not only involves an intractable integral but also mixes together the effect from different dimensions. To further analyze this problem, we start with a simplified case where $D = 1$. In this setting, we have the following:
\begin{align*}
    &\sum_{c=1}^C r_{ncD} \int \frac{q_{c'}(\omega)q_{c}(\omega)}{\sum_{c=1}^C q_c(\omega)} d\omega = \int\frac{q_{c'}(\omega)p_{n}(\omega)(f_{n,D}(\omega) - \mu_{n,D})}{\sum_{c=1}^C q_c(\omega)} d\omega,\\
    &r_{nc}(\omega) \coloneqq \frac{ p_n(\omega)\left(f_{n,D}(\omega) - \mu_{n,D}\right)r_{ncD}}{p_n^2(\omega)(f_{n,D}(\omega) - \mu_{n,D})^2}.
\end{align*}
for any $c' = 1,\dots, C$. Although solving this linear system is intractable, it indicates that $r_{ncD}$ roughly describes how $q_c(\omega)$ aligns with $p_n(\omega)(f_{n,d}(\omega) - \mu_{n,d})$ under the expectation of different $q_c'$. Therefore, 
$r_{nc}(\omega)$ can be seen as an indicator for the correlation between the current proposal $q_c(\omega)$ and $p_n(\omega)(f_{n,d}(\omega) - \mu_{n,d})$ that is normalized by the strength of $p_n(\omega)\left(f_{n,D}(\omega) - \mu_{n,D}\right)$.

For larger $D$, such concise equality involving $r_{ncd}$ is not available since the effect of different dimensions is mixed. We thus seek an heuristic approximation that not only reflects the same intuition but also becomes tractable in practice (see \cref{app:sssec:lara_weighting_functions} for practical implementations).
% As mentioned in \S\ref{ssec:lara_mis}, each proposal is mainly different from others in the dependence of query subsets. Since we simply parameterize each distribution  Suppose $q_c(\omega) = \mcN(\omega;\mbmu_c, \mathbf{I})$ with mean parameter $\mbmu_c$, we define $r_{nc}$ to be the normalized similarity between the current query $\mbq_n$ and $\mbmu_c$.

\section{Derivation for the Formulation of LARA}
\label{app:sec:derivation_lara}
In this section, we give the detailed derivation for the final expression of our estimator LARA:
\begin{align*}    
&\mathsf{LARA}\left(\mbq_{n},\mbK,\mbV\right)\\
&=\frac{\sum_{c=1}^C \alpha'_{nc}(\omega_c)  \xi(\mbq_n,\omega_c)^\top  \sum_{m=1}^M\xi(\mbk_m, \omega_c) \mbv_{m}^{\top}}{\sum_{c=1}^C \alpha'_{nc}(\omega_c)  \xi(\mbq_n,\omega_c)^\top \sum_{m=1}^M \xi(\mbk_{m}, \omega_c)}, \numberthis\label{app:eqn:lara_as_weighted_rfa}
% \frac{\sum_{s=1}^S\frac{\tilde{p}_n(\omega_s)}{q(\omega_s)} f_n(\omega_s)}{\sum_{s=1}^S\frac{\tilde{p}_n(\omega_s)}{q(\omega_s)}}.
\end{align*}
First, recall that 
\begin{align*}
p_n(\omega) &= \frac{\mathcal{N}(\omega;0,\mathbf{I})  \sum_{m=1}^M \xi(\mbq_n,\omega)^\top\xi(\mbk_{m}, \omega)}{\sum_{m'=1}^M \exp\left(\mbq_n^\top\mbk_{m'}\right)} \coloneqq \frac{\tilde{p}_n(\omega)}{Z_p}; \\
f_n(\omega) &= \frac{\sum_{m=1}^M \xi(\mbq_n,\omega)^\top \xi(\mbk_m, \omega) \mbv_{m}^{\top}}{\sum_{m'=1}^M \xi(\mbq_n,\omega)^\top \xi(\mbk_{m'}, \omega)},
\end{align*}
The formulation (\cref{app:eqn:lara_as_weighted_rfa}) is obtained by substituting the equations above into the self-normalized estimator:
\begin{align*}    
\frac{\sum_{c=1}^C\alpha_{nc}(\omega_c)\frac{\tilde{p}_n(\omega_c)}{q_c(\omega_c)} f_n(\omega_c)}{\sum_{c=1}^C\alpha_{nc}(\omega_c)\frac{\tilde{p}_n(\omega_c)}{q_c(\omega_c)}}
&=\frac{\sum_{c=1}^C\alpha_{nc}(\omega_c)\frac{\mathcal{N}(\omega_c;0,\mathbf{I})   \textcolor{keywords}{\sum_{m=1}^M\xi(\mbq_n,\omega_c)^\top\xi(\mbk_{m}, \omega_c)}}{q_c(\omega_c)} \frac{\sum_{m=1}^M \xi(\mbq_n,\omega_c)^\top \xi(\mbk_m, \omega_c) \mbv_{m}^{\top}}{\textcolor{keywords}{\sum_{m'=1}^M \xi(\mbq_n,\omega_c)^\top \xi(\mbk_{m'}, \omega_c)}}}{\sum_{c=1}^C\alpha_{nc}(\omega_c)\frac{\mathcal{N}(\omega_c;0,\mathbf{I})  \sum_{m=1}^M \xi(\mbq_n,\omega_c)^\top\xi(\mbk_{m}, \omega_c)}{q_c(\omega_c)}}\\
&=\frac{\sum_{c=1}^C\alpha_{nc}(\omega_c)\frac{\mathcal{N}(\omega_c;0,\mathbf{I})  \sum_{m=1}^M \xi(\mbq_n,\omega_c)^\top\xi(\mbk_{m}, \omega_c)\mbv_{m}^{\top}}{q_c(\omega_c)}}{\sum_{c=1}^C\alpha_{nc}(\omega_c)\frac{\mathcal{N}(\omega_c;0,\mathbf{I})  \sum_{m=1}^M \xi(\mbq_n,\omega_c)^\top\xi(\mbk_{m}, \omega_c)}{q_c(\omega_c)}}\\
&=\frac{\sum_{c=1}^C\alpha_{nc}(\omega_c)\frac{\mathcal{N}(\omega_c;0,\mathbf{I})}{q_c(\omega_c)}  \sum_{m=1}^M \xi(\mbq_n,\omega_c)^\top\xi(\mbk_{m}, \omega_c)\mbv_{m}^{\top}}{\sum_{c=1}^C\alpha_{nc}(\omega_c)\frac{\mathcal{N}(\omega_c;0,\mathbf{I})}{q_c(\omega_c)}\sum_{m=1}^M \xi(\mbq_n,\omega_c)^\top\xi(\mbk_{m}, \omega_c)}\\
&\coloneqq
\frac{\sum_{c=1}^C \alpha'_{nc}(\omega_c) \sum_{m=1}^M \xi(\mbq_n,\omega_c)^\top  \xi(\mbk_m, \omega_c) \mbv_{m}^{\top}}{\sum_{c=1}^C \alpha'_{nc}(\omega_c)  \sum_{m=1}^M \xi(\mbq_n,\omega_c)^\top \xi(\mbk_{m'}, \omega_c)}
= \mathsf{LARA}\left(\mbq_{n},\mbK,\mbV\right).
% \frac{\sum_{s=1}^S\frac{\tilde{p}_n(\omega_s)}{q(\omega_s)} f_n(\omega_s)}{\sum_{s=1}^S\frac{\tilde{p}_n(\omega_s)}{q(\omega_s)}}.
\end{align*}
Note that we define $\alpha_{nc}'(\omega_c) \coloneqq \alpha_{nc}(\omega_c)\frac{\mathcal{N}(\omega_c;0,\mathbf{I})}{q_c(\omega_c)}$.
% Therefore, for Performer-style random features, to unbiasedly estimate the term \cref{eqn:rfa-snis} that was biasedly estimated previously, we should first draw samples from a mixture of Gaussians, and then compute the averaged value accordingly. Note that this attention mechanism also takes quadratic time since for each distinct query the model has to draw samples from different distributions.

\section{Proof for the Unbiasedness of Multiple Importance Sampling}
\label{app:sec:proof_for_unbiasedness_of_mis}
As in \S\ref{ssec:lara_mis}, suppose our MIS estimator takes the following form
\begin{equation*}
\hat{g}_n = \sum_{c=1}^C \alpha_{nc}(\omega_c) \frac{p_n(\omega_c)}{q_c(\omega_c)}f_n(\omega_c), \quad \omega_c \sim q_c(\omega).
\end{equation*}
If $\sum_{c=1}^C \alpha_{nc}(\omega) = 1$, it can be shown that \citep{veach1995optimally}
\begin{align*}
    \E\left[\hat{g}_n\right] &= \sum_{c=1}^C \E_{q_c(\omega)}\left[\alpha_{nc}(\omega) \frac{p_n(\omega)}{q_c(\omega)}f_n(\omega)\right] = \sum_{c=1}^C \int \alpha_{nc}(\omega) p_n(\omega) f_n(\omega) d\omega\\
    &= \int \sum_{c=1}^C \alpha_{nc}(\omega) p_n(\omega) f_n(\omega) d\omega = \int p_n(\omega) f_n(\omega) d\omega = \E_{p_n(\omega)}\left[f_n(\omega)\right].
\end{align*}

\section{Details of RA, RFA and LARA}
\label{app:sec:detail_ra_rfa_lara}

\subsection{Sepcifics of Random Feature Attention}
\label{app:ssec:rfa}
Some implementations of RFA (including Performer \citep{choromanski2021rethinking}) defines a sample-redrawing schedule, where the involved samples $\omega$ are periodically redrawn according to a hand-crafted strategy. However, this requires a task-specific specification and we found tuning redrawing strategies only brings marginal performance gain over the simplest method that redraws samples at each training iteration (we use the same sample set during the entire evaluation phase). Therefore, we adopt this method to train Performer for all tasks. We also do not use orthogonal random samples as in \citet{choromanski2021rethinking}, as we found it does not improve empirical performance but increases the training time. \cref{alg:rfa} provides a algorithm sketch for random feature attention and linear randomized attention, respectively. Note that every loop involved in all provided pseudo-codes (\cref{alg:ra}, \cref{alg:rfa} and \cref{alg:lara}) can be trivially executed in parallel.

\subsection{Specifics of Randomized Attention}
\label{app:ssec:ra}
In this section, we describe the details of RA approximation for softmax attention. Recall in \cref{prop:ra} the RA sampling distribution is a Gaussian mixture
\begin{equation}
    p_n(\omega)
    = \sum_{m=1}^M \frac{\exp(\mbq_n^\top \mbk_{m})}{\sum_{m'=1}^M \exp(\mbq_n^\top \mbk_{m'})} \mathcal{N}(\omega; \mbq_n + \mbk_m, \mathbf{I}) \coloneqq \sum_{m=1}^M \pi_{nm} \mathcal{N}(\omega; \pmb{\mu}_{nm}, \mathbf{I}) \label{app:eqn:gaussian_mixture_ra_density}
\end{equation}
with $\pi_{nm} = \frac{\exp\left( \mbq_n^\top\mbk_m \right)}{\sum_{m'=1}^M\exp\left( \mbq_n^\top\mbk_{m'} \right)}$ and $\mbmu_{nm} = \mbq_n + \mbk_m$.
To sample from this Gaussian mixture distribution, we first sample $z_n \sim \operatorname{Categorical}(z; \mbpi_n)$ with $\mbpi_n$ being the probability masses at $M$ possible outcomes and then let $\mba_n \coloneqq [a_{n1}, \dots, a_{nM}]$ be an $M$-dimensional one-hot vector with $a_{nz_n} = 1$. The discrete random variable $\mba_n$ defines which distribution component is selected. Since all components are Gaussian, we leverage reparameterization trick \citep{vae1, vae2, vae3} to draw independent $\epsilon \sim \mcN(\omega;0, \mathbf{I})$ and add it to the selected mean, resulting in the final mixture sample. Formally, we express the sample $\omega_n$ from the Gaussian mixture as follows:
\begin{equation}
    \omega_n = \sum_{m=1}^M a_{nm} \mbmu_{nm} + \epsilon = \sum_{m=1}^M a_{nm} (\mbq_n + \mbk_m) + \epsilon = \mbK\mba_n + \mbq_n + \epsilon, \quad\epsilon \sim \mcN(\epsilon; 0, \mathbf{I}),\label{app:eqn:unbiased_ra_sampling}
\end{equation}
which is then used to compute $f_n(\omega_n)$ to obtain the RA estimation (see \cref{alg:ra} for a algorithm sketch).
Assuming the number of samples is $S$ and the sequence length is $N$, the overall time/space complexity for RA is $\mcO(SN^2)$. Through experiments we take $S=1$ sample in our randomized attention unless specified otherwise. We found this choice suffices to achieve good performance and increasing $S$ does not greatly improve the performance but introduces significant time/memory overheads.

\begin{algorithm}[H]
   \caption{Randomized Attention (RA)}
   \label{alg:ra}
\begin{algorithmic}
    \STATE {\bfseries Input:} the randomized mapping $\xi(\cdot,\cdot)$, queries $\mbQ \coloneqq \{\mbq_n\}_{n=1}^N$, keys $\mbK \coloneqq \{\mbk_m\}_{m=1}^M$, values $\mbV \coloneqq \{\mbv_m\}_{m=1}^M$ and the number of samples $S$;
    \STATE {\bfseries Output:} attention output $\mbY \coloneqq \{\mby_n\}_{n=1}^N$;
    \FOR{$n=1$ {\bfseries to} $N$}
        \FOR{$m=1$ {\bfseries to} $M$}
            \STATE Compute $\pi_{nm} \gets \frac{\exp(\mbq_n^\top \mbk_m)}{\sum_{m'=1}^M \exp(\mbq_n^\top \mbk_{m'})}$;
        \ENDFOR
        \FOR{$s=1$ {\bfseries to} $S$}
        \STATE Sample $\mba_{ns} \sim \mathrm{Categorical}(\mbpi_n)$;
        \STATE Sample $\epsilon_s \sim \mcN(\epsilon; 0, \mathbf{I})$;
        \STATE Compute $\omega_{ns} \gets \mbK \mba_{ns} + \mbq_n + \epsilon_s$;
        \STATE Compute $N_{ns} \gets \xi(\mbq_{n},\omega_{ns})\sum_{m=1}^M \xi(\mbk_{m},\omega_{ns})\mbv_{m}^{\top}$;
        \STATE Compute $D_{ns} \gets \xi(\mbq_{n},\omega_{ns})\sum_{m=1}^M \xi(\mbk_{m},\omega_{ns})$;
        \ENDFOR
        \STATE Compute $\mby_n \gets \frac{1}{S} \sum_{s=1}^S N_{ns} / D_{ns}$;
    \ENDFOR
    
    \STATE {\bfseries Return} $\mbY \coloneqq [\mby_1, \dots, \mby_N]$.
\end{algorithmic}
\end{algorithm}

\paragraph{A biased variant of randomized attention.}
Exact sampling from the mixture distribution requires us to first select a discrete component index $\mba$ from the mixture distribution and then sample from the corresponding component. Although such randomness might bring additional regularization effect, randomly selecting an index could lead to large variance and slow down training. To accelerate convergence, we also develop a biased sampling strategy from the Gaussian mixture. According to \cref{app:eqn:unbiased_ra_sampling}, the sampled one-hot vector $\mba_n$ can be approximated by its expected value $\mbpi_n$:
\begin{equation*}
    \omega'_n = \mbK\mbpi_n + \mbq_n + \epsilon, \quad\epsilon \sim \mcN(\omega;0, \mathbf{I}).
\end{equation*}
This introduces a non-negligible sampling bias in estimating the softmax attention; however, it eliminates the need to randomly draw discrete indexing vectors $\mba_n$ and reduces the variance, especially in the case of long sequences. In fact, this biased sample can be equivalently viewed as drawn from a Gaussian:
\begin{equation}
    \omega' \sim \mcN(\omega; \mbK\mbpi_n + \mbq_n, \mathbf{I}).\label{app:eqn:biased_sample_as_gaussian}
\end{equation}
Another advantage is that this formulation allows us to maintain fully deterministic during the evaluation mode, while not introducing large discrepancies from training time. Specifically, during evaluation we only pass the expectation $\mbK\mbpi_n + \mbq_n$ as the ``sample'', which is a standard practice similar to the usage of Dropout \citep{srivastava14dropout}. This is in contrast to unbiased RA sampling, which has to draw random indices even during evaluation (otherwise, replacing both random variables with their expected values would lead to larger discrepancies between training and testing, resulting in inferior performance). 
Same as the case of RFA, we also redraw random samples at every training iteration. Note that this can not be transferred to Performer since the expectation of $\omega$ in RFA is 0, which leads to degeneration.

As a proof-of-concept experiment, we run randomized attention with biased sampling strategy on image classification with \imagenet dataset, video recognition with \kinetics and \ssv datasets and machine translation with \wmt dataset. From \cref{app:tb:unbiased_vs_biased_ra}, we note that biased RA performs better than both its unbiased counterpart for visual tasks, which usually deal with longer sequences (196 for images and 1568 for videos); but it performs worse in machine translation, where either the source or target sentence only consist of dozens of tokens. On the other hand, RA outperforms softmax attention on both image and language tasks, indicating that the proposed estimation methods for softmax attention may enjoy better modeling capacity. This may shed light on some latent mechanism in such approximation that deviates from the standard softmax attention but does better in modeling the sequence representations. We leave detailed investigation in future work.
\begin{table}[t]
	\caption{Experimental results with exact or biased randomized attention mechanisms.}
	\label{app:tb:unbiased_vs_biased_ra}
    \vskip 0.15in
	\centering
	\resizebox{\columnwidth}{!}{    
\begin{tabular}{l | c || c | c || c | c || c}
\toprule[.1em]
\multirow{3}{*}{Model} & \multirow{3}{*}{Complexity} & \multicolumn{2}{c||}{Image} & \multicolumn{2}{c||}{Video} & {Language} \\
& & \makecell[c]{Top-1 Acc. on \imagenet \\ w/ DeiT-Tiny} & \makecell[c]{Top-1 Acc. on \imagenet \\ w/ DeiT-Small} & {Top-1 Acc. on \kinetics} & {Top-1 Acc. on \ssv} & {BLEU on \wmt} \\
% 			\cmidrule{4-7}
\midrule
RA-unbiased &Quadratic  &  71.86 & 80.04 & 78.2 & 64.9 & \textbf{27.8}\\
RA-biased &Quadratic  & \textbf{72.98} & \textbf{80.49}& 79.0 & 65.9 & 27.3\\
\midrule
Softmax  & Quadratic & 72.20 & 79.90& \textbf{79.2} & \textbf{66.5} & 27.5\\
\bottomrule[.1em]
\end{tabular}}  
\vskip -0.1in
\end{table}

\subsection{Specifics of Linear Randomized Attention}
\label{app:ssec:lara}
In this section, we provide more implementation details of linear randomized attention.
\subsubsection{On the Formulation of Proposal Distributions}
\label{app:sssec:lara_proposal_form}
As mentioned in \S\ref{ssec:lara_mis}, each proposal $q_c(\omega)$ is defined to depend on some subset of queries; and their union covers the whole set of queries. Since our goal is let these proposals behave similarly to the true RA distribution $p_n(\omega)$, a straightforward choice is to specify $q_c$ as the same formulation of $p_n(\omega)$ (\cref{app:eqn:gaussian_mixture_ra_density}):
%  \coloneqq \sum_{m=1}^M \pi_{cm} \mathcal{N}(\omega; \pmb{\mu}_{cm}, \mathbf{I})
\begin{equation}
    q_c(\omega)
    = \sum_{m=1}^M \frac{\exp(\widetilde{\mbq}_c^\top \mbk_{m})}{\sum_{m'=1}^M \exp(\widetilde{\mbq}_c^\top \mbk_{m'})} \mathcal{N}(\omega; \widetilde{\mbq}_c + \mbk_m, \mathbf{I}).\label{app:eqn:q_c_as_a_gmm}
\end{equation}
Here we divide the input query sequence $\{\mbq_n\}_{n=1}^N$ into $C$ segments and compute the average (called \emph{landmarks}, the number of which is equal to the number of samples) over queries $\{\widetilde{\mbq}_c\}_{c=1}^C$ within the same segment. In particular, supposing $N$ is divisible by $C$ and $T \coloneqq N / C$ is the segment length, each segment landmark can be expressed as
\begin{equation*}
    \widetilde{\mbq}_c = \frac{1}{T} \sum_{t=1}^T \mbq_{(c-1)T+t}.
\end{equation*}
% The main difference between the $c$-th proposal and RA distribution lies in the representation for the ``query'': for proposal $q_c$, the segment average  representation is used; while the RA distribution depends on exactly each query. 
We then use each of these proposals to estimate the target expectation for the $n$-th query and combine their results into the final estimation.
However, this choice involves $CM$ distributions in total ($C$ proposals are maintained, each of which is again a Gaussian mixture with $M$ components) and sampling from these distributions may introduce large noise. Motivated by the discussion of biased sampling in RA (\cref{app:eqn:biased_sample_as_gaussian} in \cref{app:ssec:ra}), we explore an alternative parameterization by defining each proposal as a Gaussian:
\begin{equation}
    q_c(\omega)
    = \mathcal{N}(\omega; \mbK\mbpi_c + \widetilde{\mbq}_c, \mathbf{I}) = \mathcal{N}\left(\omega; \widetilde{\mbq}_c + \sum_{m=1}^M \frac{\exp(\widetilde{\mbq}_c^\top \mbk_{m})}{\sum_{m'=1}^M \exp(\widetilde{\mbq}_c^\top \mbk_{m'})} \mbk_m, \mathbf{I}\right).\label{app:eqn:q_c_as_a_gaussian}
\end{equation}
We find this choice performs better than the mixture formulation (\cref{app:eqn:q_c_as_a_gmm}) empirically. Intuitively, this strategy aggregates the information from all keys based on the correlation between the query landmarks and each individual key. However, this introduces additional $\mathcal{O}(CM)$ computational costs.

In practice, we observe that for proposal landmark $\widetilde{\mbq}_c$, keys belonging to the same segment $c$ often contribute the most to the Gaussian mean. As a result, we develop another variant that also computes the key landmarks,
\begin{equation*}
    \widetilde{\mbk}_c = \frac{1}{T} \sum_{t=1}^T \mbk_{(c-1)T+t},
\end{equation*}
and then simply let
\begin{equation}
    q_c(\omega)
    = \mathcal{N}(\omega; \widetilde{\mbq}_c + \widetilde{\mbk}_c, \mathbf{I}).\label{app:eqn:q_c_as_local_means}
\end{equation}
We observe this formulation works equally well; such parameterization is thus used throughout our experiments by default. 

\paragraph{An improved proposal parameterization for vision transformers.}
Comparing \cref{app:eqn:q_c_as_a_gaussian} and \cref{app:eqn:q_c_as_local_means}, we observe that for the former it only biases the Gaussian mean towards the direction of the current query landmark; while for the latter it only promotes information from key vectors that are in the same segment as $\widetilde{\mbq}_c$ and ignores the global information of keys. Noticing these differences, we further propose a variant bridging these two formulations:
\begin{equation}
    q_c(\omega)
    = \mathcal{N}\left(\omega; \widetilde{\mbq}_c + \sum_{c'=1}^C \frac{\exp(\widetilde{\mbk}_c^\top \widetilde{\mbk}_{c'})}{\sum_{c'=1}^M \exp(\widetilde{\mbk}_c^\top \widetilde{\mbk}_{c'})} \widetilde{\mbk}_c, \mathbf{I}\right).\label{app:eqn:q_c_as_mixed}
\end{equation}
Intuitively, this performs an attention-like aggregation operation over key landmarks. The aggregation procedure not only computes the correlation between key vectors, which alleviates the bias of being closer to query landmarks, but also collects global information while still favoring local segments. In addition, it runs with $\mathcal{O}(C^2)$, which is much cheaper than $\mathcal{O}(CM)$. We find this yields better predictive performance in vision transformers, but improves marginally for other tasks. We hypothesize that this is because the attention-like operation smooths the Gaussian mean, which aligns with that ViT tends to produce smoothed patch representations. We leave in-depth investigation as future work. In summary, we adopt this parameterization only through experiments on image classification (\S\ref{ssec:image_classification}).

% \paragraph{Discussion on incorporating information from value vectors.}
% Most methods that attempt to approximate the softmax attention rarely consider the information from value vectors.
% \begin{itemize}
%     \item (Optional) add analysis on optimal proposal distributions.
% \end{itemize}
See \cref{alg:lara} for a algorithm sketch of LARA. 

\begin{minipage}{0.47\textwidth}
\begin{algorithm}[H]
   \caption{Random Feature Attention (RFA)}
   \label{alg:rfa}
\begin{algorithmic}
    \STATE {\bfseries Input:} the randomized mapping $\xi(\cdot,\cdot)$, queries $\mbQ \coloneqq \{\mbq_n\}_{n=1}^N$, keys $\mbK \coloneqq \{\mbk_m\}_{m=1}^M$, values $\mbV \coloneqq \{\mbv_m\}_{m=1}^M$ and the number of samples $S$;
    \STATE {\bfseries Output:} attention output $\mbY \coloneqq \{\mby_n\}_{n=1}^N$;
    \STATE
    \FOR{$s=1$ {\bfseries to} $S$}
        \STATE 
        \STATE Sample $\omega_s \sim \mcN(\omega; 0, \mathbf{I})$;
        \STATE Compute $N_s \gets \sum_{m=1}^M \xi(\mbk_{m},\omega_s)\mbv_{m}^{\top}$;
        \STATE Compute $D_s \gets \sum_{m=1}^M \xi(\mbk_{m},\omega_s)$;
    \ENDFOR
    \FOR{$n=1$ {\bfseries to} $N$}
        \STATE 
        \STATE
        \STATE Compute $N \gets \sum_{s=1}^S \xi(\mbq_{n},\omega_s)N_s$;
        \STATE Compute $D \gets \sum_{s=1}^S \xi(\mbq_{n},\omega_s)D_s$;
        \STATE Compute $\mby_{n} \gets N / D$;
    \ENDFOR
    \STATE {\bfseries Return} $\mbY \coloneqq [\mby_1, \dots, \mby_N]$.
\end{algorithmic}
\end{algorithm}
\end{minipage}
\hfill
\begin{minipage}{0.47\textwidth}
\begin{algorithm}[H]
   \caption{Linear Randomized Attention (LARA)}
   \label{alg:lara}
\begin{algorithmic}
    \STATE {\bfseries Input:} the randomized mapping $\xi(\cdot,\cdot)$, queries $\mbQ \coloneqq \{\mbq_n\}_{n=1}^N$, keys $\mbK \coloneqq \{\mbk_m\}_{m=1}^M$, values $\mbV \coloneqq \{\mbv_m\}_{m=1}^M$ and the number of samples $C$;
    \STATE {\bfseries Output:} attention output $\mbY \coloneqq \{\mby_n\}_{n=1}^N$;
    \STATE Compute proposal parameters $\{\mu_c\}_{c=1}^C$;
    \FOR{$c=1$ {\bfseries to} $C$}
        \STATE Let $q_c(\omega) \gets \mcN(\omega; \mu_c, \mathbf{I})$;
        \STATE Sample $\omega_c \sim q_c(\omega)$;
        \STATE Compute $N_c \gets \sum_{m=1}^M \xi(\mbk_{m},\omega_c)\mbv_{m}^{\top}$;
        \STATE Compute $D_c \gets \sum_{m=1}^M \xi(\mbk_{m},\omega_c)$;
    \ENDFOR
    \FOR{$n=1$ {\bfseries to} $N$}
        \STATE Compute $\alpha_{nc}(\omega_c)$ according to \cref{eqn:lara:opt_weighting_function};
        \STATE Compute $\alpha'_{nc}(\omega_c) \gets \alpha_{nc}(\omega_c) \mcN(\omega_c;0,\mathbf{I}/ q_c(\omega_c)$;
        \STATE Compute $N \gets \sum_{c=1}^C \alpha'_{nc}(\omega_c)\xi(\mbq_{n},\omega_c)N_c$;
        \STATE Compute $D \gets \sum_{c=1}^C \alpha'_{nc}(\omega_c)\xi(\mbq_{n},\omega_c)D_c$;
        \STATE Compute $\mby_{n} \gets N / D$;
    \ENDFOR
%   \IF{$x_i > x_{i+1}$}
%   \STATE Swap $x_i$ and $x_{i+1}$
    %   \STATE $noChange = false$
    %   \ENDIF
    \STATE {\bfseries Return} $\mbY \coloneqq [\mby_1, \dots, \mby_N]$.
\end{algorithmic}
\end{algorithm}
\end{minipage}

% Intuitively, to be close to the true distribution $p_n(\omega)$, each proposal $q_c(\omega)$ should be a Gaussian mixture and depend on some subset of queries and the whole set of keys. However, we empirically find sampling from a mixture introduces additional noise (see \cref{app:ssec:randomized_attention} for details) and performs worse than sampling from a single Gaussian. To this end, we define all of the proposals to be Gaussian $\mcN(\mbmu_c, \mathbf{I})$ across our experiments. We divide the input sequence into $C$ segments and compute the mean of each segment for both queries and keys, denoted by $\bar{\mbq}_c$ and $\bar{\mbk}_c$ respectively. We further perform a attention-like aggregation over all key segment means $\{\bar{\mbk}_c\}_{c=1}^C$ to collect global key information. Denoting the final representation of query and key segments by $\tilde{\mbq}_c$ and $\tilde{\mbk}_c$, the parameter $\mbmu_c$ of each proposal is then specified as $\mbmu_c = \tilde{\mbq}_c + \tilde{\mbk}_c$.

    % \frac{q_c(\omega_c)}{\sum_{c'=1}^C q_{c'}(\omega_c)} + q_c(\omega_c)\left(r_{nc}(\omega_c) - \sum_{c=1}^C\frac{q_c(\omega_c)}{\sum_{c'=1}^C q_{c'}(\omega_c)}r_{nc}(\omega_c)\right).
\subsubsection{On the Parameterization of Weighting Functions}
\label{app:sssec:lara_weighting_functions}
Our MIS estimating strategy introduces a set of weighting functions $\alpha(\cdot)$ for each proposal. A common choice of weighting functions in MIS \citep{mcbook} is the \emph{balance heuristic} strategy
\begin{equation}
    \alpha_c(\omega_c) = \frac{q_c(\omega_c)}{\sum_{c'=1}^C q_{c'}(\omega_c)},
\end{equation}
which is nearly optimal in that any other weighting schemes will not exhibit significantly smaller variance \citep{veach1995optimally}.
However, this strategy only considers the relative strengths of proposals and ignores contextual information from each query. 
As a result, a na\"ive application of MIS would disregard the inherent variation among different queries and fails to describe the specialized target distribution $p_n(\omega)$. 

Instead of balance heuristics, we adopt query-specific weighting functions that are inspired by query-optimal analysis. In our MIS scheme (\cref{eqn:lara:opt_weighting_function}), the optimal weighting functions take the following form
\begin{align*}
    \alpha^*_{nc}(\omega_c) =
    {\underbrace{\frac{q_c(\omega_c)}{\sum_{c'=1}^C q_{c'}(\omega_c)}\vphantom{q_c(\omega_c)\left(r_{nc}(\omega_c) - \sum_{c=1}^C\frac{q_c(\omega_c)}{\sum_{c'=1}^C q_{c'}(\omega_c)}r_{nc}(\omega_c)\right)}}_\text{balance heuristic}} + {\underbrace{q_c(\omega_c)\left(r_{nc}(\omega_c) - \sum_{c=1}^C\frac{q_c(\omega_c)}{\sum_{c'=1}^C q_{c'}(\omega_c)}r_{nc}(\omega_c)\right)}_\text{query-specific correction}}
\end{align*}
Note that it sums to 1 over all $c$'s and is a valid weighting function:
\begin{align*}
    \sum_{c=1}^C \alpha^*_{nc}(\omega_c) &=
    \sum_{c=1}^C \frac{q_c(\omega_c)}{\sum_{c'=1}^C q_{c'}(\omega_c)} + \sum_{c=1}^C q_c(\omega_c)\left(r_{nc}(\omega_c) - \sum_{{c^{''}}=1}^C\frac{q_{c^{''}}(\omega_{c^{''}})}{\sum_{c'=1}^C q_{c'}(\omega_c)}r_{n{c^{''}}}(\omega_{c^{''}})\right) \\
    &= 1 + \left[\sum_{c=1}^C q_c(\omega_c)r_{nc}(\omega_c) - \left(\sum_{c=1}^C q_c(\omega_c)\right) \sum_{{c^{''}}=1}^C\frac{q_{c^{''}}(\omega_{c^{''}})}{\sum_{c'=1}^C q_{c'}(\omega_c)}r_{n{c^{''}}}(\omega_{c^{''}})\right] \\
    &= 1 + \left[\sum_{c=1}^C q_c(\omega_c)r_{nc}(\omega_c) -  \sum_{c=1}^Cq_{c}(\omega_{c})r_{n{c}}(\omega_{c})\right] \\
    &= 1 + 0 = 1.
\end{align*}
In particular, we observe the first term is the ordinary balance heuristic weighting function, while the second term is a query-specific correction that sums to 0.

As mentioned in \cref{app:sec:opt}, the exact form of $r_{nc}(\cdot)$ is mostly intractable to compute in practice. To this end, we introduce a heuristic yet tractable $r'_{nc}$ to roughly align with the intuition of original $r_{nc}(\cdot)$:
\begin{align*}
    \alpha_{nc}(\omega_c) &= {\underbrace{\frac{q_c(\omega_c)}{\sum_{c'=1}^C q_{c'}(\omega_c)}\vphantom{q_c(\omega_c)\left(r_{nc}(\omega_c) - \sum_{c=1}^C\frac{q_c(\omega_c)}{\sum_{c'=1}^C q_{c'}(\omega_c)}r_{nc}(\omega_c)\right)}}_\text{balance heuristic}} + {\underbrace{q_c(\omega_c)\left(r'_{nc} - \sum_{c=1}^C\frac{q_c(\omega_c)}{\sum_{c'=1}^C q_{c'}(\omega_c)}r'_{nc}\right)}_\text{query-specific correction}} \numberthis\label{app:eqn:weighting_coupled}\\
    r'_{nc} &= \frac{\exp{(\mbq_n^\top\tilde{\mbq}_c)}}{\sum_{n=1}^N \exp{(\mbq_n^\top\tilde{\mbq}_{c'})}},
\end{align*}
Intuitively, we implement $r'_{nc}$ as the normalized similarity between the $n$-th query and the $c$-th segment-averaged query vector.
In addition, we note that the query-specific information $r'_{nc}$ is influenced by the query-agnostic density $q_c$, which may be incorrectly suppressed or amplified if the drawn sample lies in a low-density region. Base on this, we further propose a simplified formulation:
\begin{align*}
    \alpha_{nc}(\omega_c) &= {\underbrace{\frac{q_c(\omega_c)}{\sum_{c'=1}^C q_{c'}(\omega_c)}\vphantom{r'_{nc} - \frac{1}{C}\sum_{c=1}^Cr'_{nc}}}_\text{balance heuristic}} + {\underbrace{r'_{nc} - \frac{1}{C}\sum_{c=1}^Cr'_{nc}}_\text{query-specific correction}}.\numberthis\label{app:eqn:weighting_decoupled}
\end{align*}
where we decouple the computation between proposal densities $q_c(\cdot)$ and $r'_{nc}$. In this way, query-dependent and query-agnostic information will be independent of each other.

We also notice that the query-specific information can be explicitly controlled by introducing a coefficient $\beta$ such that
\begin{align*}
    \alpha_{nc}(\omega_c) &= {\underbrace{\frac{q_c(\omega_c)}{\sum_{c'=1}^C q_{c'}(\omega_c)}\vphantom{\beta\left(r'_{nc} - \frac{1}{C}\sum_{c=1}^Cr'_{nc}\right)}}_\text{balance heuristic}} + {\underbrace{\beta\left(r'_{nc} - \frac{1}{C}\sum_{c=1}^Cr'_{nc}\right)}_\text{correction}}.
\end{align*}
This weighting function remains valid since the correction term still sums to 0. By setting $\beta > 1$, the mechanism tends to favor the query-specific information over the balance heuristic. We tried several choices of $\beta$ and found $\beta = 2$ slightly improves the performance. As reflected in our ablation study, we demonstrate the superior performance of query-specific weighting functions over vanilla balance heuristics \citep{veach1995optimally}.

\subsubsection{Training and Evaluation Details}
LARA redraws samples from proposal sets at every training iteration; during evaluation, we simply pass corresponding expected values instead of drawing samples, in a similar way to dropout \citep{srivastava14dropout}.

\subsubsection{Complexity Analysis}
Recall there are $N$ queries and $M$ key-value pairs. Like RFA, the involved computation of our LARA estimator includes (1) computing the proposal distribution, which may take $\mathcal{O}(C)$ or $\mathcal{O}(C^2)$ time (\cref{app:sssec:lara_proposal_form}); (2) a pre-computing step over all key-value statistics, which takes $\mathcal{O}(CM)$ time and space; and (3) applying pre-computed statistics to all queries, taking $\mathcal{O}(CN)$ complexity. These steps result in overall $\mathcal{O}(CM + CN)$ complexity given $C \ll \min (M, N)$. Note that $C$ is analogous to the number of samples $S$ (often referred to as random feature dimension \citep{choromanski2021rethinking}) in RFA. Therefore, LARA does not incur a heavy computational overhead compared to RFA, as also reflected in \S\ref{ssec:running_time_and_memory}.

\section{Additional Experimental Details}
\label{app:sec:experiment_details}

\subsection{Preliminary Experiments on Approximation Quality}
\label{app:ssec:toy}
We conduct the preliminary experiment on vision transformers (ViT), which first split input images into small patches, serialize them as a 1D sequence and then processes the sequence through a transformer model. To be specific, we replace the standard softmax attention in vision transformers \citep[ViT;][]{dosovitskiy2021vit,touvron21adeit} with different approximation variants. The MSE is evaluated under three different sequence lengths $N$ by varying the image resolution and patch size: (a) resolution 224 x 224 with patch size 16 ($N=196$), (b) resolution 384 x 384 with patch size 16 ($N=576$) and (c) resolution 224 x 224 with patch size 8 ($N=784$). To achieve a fair comparison, we use pretrained ViTs whose weights are trained under corresponding sequence lengths with \emph{standard softmax attention}. The sequence length is selected according to whether the ViT weights pretrained by softmax attention are available. Since there are multiple attention blocks in ViT architecture, for each input image we average the attention MSE over all attention heads and transformer layers.

\subsection{Image Classification}
\label{app:ssec:image}
For image classification, we consider two vision transformer architectures: vanilla ViT \citep{dosovitskiy2021vit} and PVTv2 \citep{pvtv2}. We refer to ViT as DeiT \citep{touvron21adeit} through this work, since DeiT follows the same model architecture as ViT but adopts greatly improved training protocols.

\paragraph{Details of applying LARA to DeiT.}
We do not use the distillation technique as in DeiT \citep{touvron21adeit}. We following the same procedure to train DeiT on \imagenet dataset as in \citet{touvron21adeit}. In particular, we use AdamW optimizer \citep{adamw} for 300 epochs, where we set the batch size to 1024 and the learning rate to 0.001 with cosine learning rate decay \citep{cos-lr}. The number of warm-up epochs is set to 10 for all models instead of 5, since we find it often stabilizes training and leads to better results. For data augmentation, we follow \citet{touvron21adeit} and use random clipping, cropping, rand-augment \citep{random-augment} and random erasing \citep{random-erasing}. We remove repeated augmentation \citep{repeat-augment} as it often slows down convergence, as also observed in previous studies \citep{berman2019multigrain,xiao2021early}. For regularization, we employ stochastic depth \citep{stochastic-depth}, Mixup \citep{mixup}, Cutmix \citep{cutmix}, label smoothing and weight decay, all of which are set to default settings in DeiT \citep{touvron21adeit}. Unless otherwise specified, the input image size is set to $224 \times 224$ with patch size $16$, resulting in $14 \times 14 = 196$ non-overlapping patch tokens. 
For LARA in DeiT models, we additionally transform the average query/key vector of each segment through a fully connected layer followed by a layer-norm operation. This corresponds to importance sampling with adaptive proposals \citep{mcbook}, which improves the expressiveness of the proposal distributions. Note that the linear transformation is shared among all attention heads, which results in only marginal additional computational overheads.

\paragraph{Details of applying LARA to PVTv2.}
Pyramid Vision Transformers v2 \citep[PVTv2;][]{pvtv2} is a strong vision transformer baseline with pyramidal architectures that processes much longer token sequences. It first patchifies input images into a $56\times56$ token sequence, which is then processed by 4 successive stages. Each stage consists of a stack of transformer layers and processes the input sequence from the previous stage by reducing both the height and width of patch tokens to the half and increasing the embedding dimension by a factor of 2. The detailed configuration for all model sizes follows Table~1 of \citet{pvtv2}. In such architecture, the sequence at early stages is too long to be handled by regular softmax attention. To address this issue, PVTv2 proposes an efficient variant Spatial Reduction Attention (SRA) and uses SRA for all attention blocks in the first three stages and ordinary softmax attention for the last stage due to reduced resolution. For each SRA module, it use a convolutional layer to reduce the length of input sequence to 49, which is then projected to key and value vectors correspondingly. The query set maintains the same resolution and performs attention over the shortened key-value sequence to obtain globally contextualized representations. 

To evaluate our method on PVTv2, we replace all SRA modules with either Performer or LARA. For PVTv2 with Performer, we use 128 samples since it fails to converge with fewer samples. In terms of PVTv2 with LARA, we do not use convolutional blocks and simply use 2D average pooling (the same as segments) followed by a linear projection to obtain query and key landmarks, the number of which is set to 49 as in SRA. Since we do not use the convolutional block, Both Performer and LARA use much fewer model parameters than vanilla PVTv2.

In addition, vanilla PVTv2 model uses 1,2,5 and 8 attention heads for its 4 processing stages respectively in its original implementation; however, we found using 2$\times$ more heads consistently improves predictive performance for all methods (including baseline SRA, Performer and LARA) while introducing affordable overheads. Therefore, we use 2,4,10 and 16 heads for all PVTv2-based models across our experiments. We mostly follow the training protocol as \citet{pvtv2} to train all PVTv2-based models, except that we increase the number of warm-up epochs from 5 to 10. We find a slightly longer warm-up schedule is helpful to improve the model performance. 

\subsection{Video Action Recognition}
\label{app:ssec:video}
% \paragraph{Note on the size of K400 dataset.}
% The number of videos for training or validation in \kinetics might be different from previous work that also use this dataset. Due to all the videos in the \kinetics dataset needs to be downloaded manually from Youtube, we are only able to access a subset of the original training/validation video clips, since some of them may be taken offline and are no longer available to download.

Our implementation is based on the \texttt{PySlowFast} \citep{fan2020pyslowfast} codebase and we follow the training protocol in Motionformer \citep{patrick2021motionformer}. In particular, Motionformer adopts the vision transformer base (ViT/B) \citep{dosovitskiy2021vit} architecture which has 12 transformer encoder layers with 12 attention heads and 768-dimensional hidden representations. For \kinetics dataset, its parameter weights are pretrained on \imagenetfull dataset with regular softmax attention; while for \ssv dataset, we use the trained weights on \kinetics with the corresponding attention variant. The model operates on videos with size $16 \times 224 \times 224$, which is then split into $8 \times 14 \times 14$ tubes with separate space-time positional embedding. 
Motionformer introduces the trajectory attention, which first computes spatial attention to obtain probabilistic trajectories, which are then aggregated temporally. We use the trajectory attention module \citep{patrick2021motionformer} and replace the involved softmax attention with different attention approximation methods. 
Besides Performer \citep{choromanski2021rethinking}, Nystr\"omformer \citep{xiong2021nystromformer} and full trajectory attention, our baselines also include Orthoformer \citep{patrick2021motionformer}, another strong baseline for video transformers that constructs a low-rank approximation of attention matrix via sequentially selecting orthogonal query landmarks. For all efficient attention variants, we set both the number of samples (in LARA and Performer) or the number of landmarks (in Nystr\"omformer and Orthoformer) to 128 for a fair comparison.
For data augmentation, we also follow \citet{patrick2021motionformer}, adopting random scale jittering, random horizontal flips and color jittering for all datasets; and additionally rand-augment \citep{random-augment} for \ssv dataset. 

We use the AdamW \citep{adamw} optimizer to train LARA for 40 epochs with weight decay 0.05, label smoothing rate 0.2 and total batch size 256. A slightly longer training schedule (compared to 35) is adopted since our method involves additional randomness and we found training for a longer time improves convergence. The initial learning rate is set to 0.0001 and gets decayed by a ratio of 10 at epochs 25 and 35 respectively. During training, the video clips are randomly sampled with cropped resolution $224 \times 224$; while for testing, we sample 10 uniform temporal clips per video with 3 spatial crops per clip and average scores for these crops to obtain the final prediction.

% As a result, this setting also allows us to access how the discrepancy between softmax attention and its approximations would influence training. 

\subsection{Machine Translation}
\label{app:ssec:mt}
We use the Transformer-base architecture as specified in \citet{vaswani2017attention} for our machine translation experiments. The model contains a transformer encoder and decoder, both of which consist of 6 layers with hidden size and number of heads being 512 and 8, respectively. The vocabulary is shared between source and target language, consisting of around 32K byte pair encoding \citep[BPE;][]{sennrich-etal-2016-bpe} types. The hidden dimension of feed forward networks is set to 2048. The rate of dropout is set to 0.1. As mentioned in the main paper, we only replace encoder self-attention in transformer models with efficient attention variants. Since LARA does not support causal attention mode in its current version, this setting allows us to directly assess the ability of different attention mechanisms to learn contextualized representations. Recent studies also indicate that in neural machine translation the transformer encoder seems playing a more important role in extracting representations \citep{kasai2021deepNMT}. 
Besides Performer, we also compare our method against other baselines including (1) Linformer \citep{wang2020linformer}, which is widely adopted in the context of NLP, (2) ABC \citep{peng2021abc}, a recently proposed unified framework of most low-rank attention approximations and (3) Nystr\"omformer \citep{xiong2021nystromformer}, which we find is a strong low-rank baseline across our experiments. 

For training, we follow the same setup as in \citet{vaswani2017attention}. In particular, we use the Adam optimizer \citep{kingma2014adam} with learning rate 0.0007, label smoothing rate 0.1, inverse square root learning rate scheduler and 4000 warm-up steps. During decoding, we set beam size to 4, length penalty to 0.6, average last 10 checkpoints and apply a compound split post-processing to facilitate comparison.

\subsection{Efficiency Analysis}
\label{app:ssec:time_mem}
For the simulation experiment conducted in \S\ref{ssec:running_time_and_memory}, the same transformer architecture is used for all attention methods, which consists of 8 encoder layers with 192 embedding dimension and 3 attention heads. The use of smaller-size transformer model allows us to run longer lengths for softmax attention and randomized attention. The detailed running time (in ms) and memory consumption is listed in \cref{app:tb:time_mem}. For Nystr\"omformer \citep{xiong2021nystromformer} and Linformer \citep{wang2020linformer}, the number of landmarks is set to 16; for Performer and LARA, the number of samples is set to 16 as well.
\begin{table}[t]
	\caption{Empirical running time and memory consumption for different attention mechanisms. We report the absolute time in millisecond and memory usage in GB; the relative time/memory cost to the softmax attention is also reported in brackets.}
	\label{app:tb:time_mem}
	\vskip 0.15in
	\centering
	\resizebox{\columnwidth}{!}{    
		\begin{tabular}{l || c c c c c || c c c c c}
            \toprule[.1em] 
            
            \multirow{2}{*}{Models} & \multicolumn{5}{c||}{Running Time (ms)} & \multicolumn{5}{c}{Peak Memory Usage (GB) } \\
            & 1024 & 2048 & 3072 & 4096 & 8192 & 1024 & 2048 & 3072 & 4096 & 8192 \\
            \midrule 
            Full softmax & 9.44 (1.00$\times$) & 19.96 (1.00$\times$) & 42.45 (1.00$\times$) & 69.25 (1.00$\times$) & 271.12 (1.00$\times$) & 0.11 (1.00$\times$) & 0.33 (1.00$\times$) & 0.68 (1.00$\times$) & 1.18 (1.00$\times$) & 4.58 (1.00$\times$) \\
            Nystr\"omformer & 37.51 (3.98$\times$) & 36.90 (1.85$\times$) & 37.44 (0.88$\times$) & 37.21 (0.54$\times$) & 38.36 (0.14$\times$) & 0.05 (0.45$\times$) & 0.06 (0.20$\times$) & 0.08 (0.12$\times$) & 0.10 (0.08$\times$) & 0.16 (0.03$\times$)\\
            Linformer & 12.85 (1.36$\times$) & 12.57 (0.63$\times$) & 12.57 (0.30$\times$) & 12.62 (0.18$\times$) & 19.14 (0.07$\times$) & 0.05 (0.46$\times$) & 0.07 (0.20$\times$) & 0.08 (0.12$\times$) & 0.10 (0.09$\times$) & 0.17 (0.04$\times$) \\
            Performer & 18.74 (1.99$\times$) & 18.80 (0.94$\times$) & 19.04 (0.45$\times$) & 18.89 (0.27$\times$) & 33.00 (0.12$\times$) & 0.05 (0.44$\times$) & 0.06 (0.19$\times$) & 0.08 (0.11$\times$) & 0.09 (0.08$\times$) & 0.15 (0.03$\times$) \\
            \midrule
            LARA & 19.91 (2.11$\times$) & 19.82 (0.99$\times$) & 19.87 (0.47$\times$) & 19.91 (0.29$\times$) & 31.81 (0.12$\times$) & 0.06 (0.51$\times$) & 0.08 (0.24$\times$) & 0.10 (0.15$\times$) & 0.12 (0.10$\times$) & 0.19 (0.04$\times$) \\
            RA & 17.09 (1.81$\times$) & 53.56 (2.68$\times$) & 108.17 (2.55$\times$) & 183.03 (2.64$\times$) & 756.72 (2.79$\times$) & 0.23 (2.10$\times$) & 0.80 (2.45$\times$) & 1.74 (2.55$\times$) & 3.06 (2.59$\times$) & 12.10 (2.64$\times$) \\
            \bottomrule[.1em]
\end{tabular}}  
	\vskip -0.1in
\end{table}

\section{Additional Experimental Results}
\label{app:sec:experiment_results}

\subsection{Additional Experiments on Image Classification}
\label{app:ssec:additional_imagenet}
We conduct additional experiments to evaluate the performance of our proposed method at various aspects. First, we vary the number of random samples to investigate the effect of sample size on the performance for \imagenet dataset. As presented in \cref{app:tb:deit_wrt_samples}, although both Performer and LARA improves predictive accuracy as the number of samples increases, LARA benefits much more than Performer, and finally outperforms softmax attention with 196 samples, which is equal to the sequence length.

In addition, we also compare LARA against different efficient attention mechanisms, as shown in \cref{app:tb:deit_wrt_attn}. We note that LARA outperforms most efficient attention mechanisms by a large margin, and slightly outperforms Nystr\"omformer \citep{xiong2021nystromformer}, which we found is a strong baseline across various domains.
% Gaussian , mu = q 72.40% tiny, 80.16 small

\begin{minipage}{\textwidth}
\begin{minipage}[c]{0.4\textwidth}
\centering
\captionsetup{type=table}
\caption{Top-1 accuracy (\%) on \imagenet validation set for Performer and Lara at different numbers of samples on DeiT-Small.}
\label{app:tb:deit_wrt_samples}
\vskip 0.15in
\resizebox{0.8\columnwidth}{!}{    
	\begin{tabular}{l | c | c | c | c}
		\toprule[.1em]
		\multirow{2}{*}{Model} & \multicolumn{4}{c}{\# Samples} \\
		        & 25 & 49 & 100 & 196 \\
		\midrule
		Performer  & 73.37 & 73.63 & 74.15 & 74.44  \\
		\midrule
		LARA  & 78.29  & 79.48 & 79.89 & 80.57\\
		\midrule
		RA & \multicolumn{4}{c}{80.04} \\
		\midrule
        Softmax & \multicolumn{4}{c}{79.90}\\
		\bottomrule[.1em]
\end{tabular}}  
\vskip -0.1in
\end{minipage}
\hfill
\begin{minipage}[c]{0.5\textwidth}
\centering
\captionsetup{type=table}
\caption{Top-1 accuracy (\%) on \imagenet validation set with DeiT-Small under different attention mechanisms.}
\label{app:tb:deit_wrt_attn}
\vskip 0.15in
\resizebox{0.8\columnwidth}{!}{    
	\begin{tabular}{l | c }
		\toprule[.1em]
		Model & {Top-1 Acc.}\\
		\midrule
		Performer \citep{choromanski2021rethinking} &  74.3\\
		SRA (Convolutional) \citep{pvt,pvtv2} &  74.4 \\ 
		Linformer \citep{wang2020linformer} &  76.0 \\
		XCIT \citep{el2021xcit} &  77.9 \\
		Nystr\"omformer \citep{xiong2021nystromformer} &  79.3 \\
		LARA & \textbf{79.5} \\
		\midrule
		Softmax attention  &  79.9\\
		\bottomrule[.1em]
\end{tabular}}  
\end{minipage}
\end{minipage}

\subsection{Additional Experiments on Long Range Arena Benchmark}
\label{app:ssec:additional_lra}
% \subsection{Long Range Arena Benchmark}
% \label{ssec:long_range_arena}
We also evaluate our model on the Long Range Arena (\lra) benchmark \citep{tay2021long}, which is designed to test the ability to process long sequences and generalize over diverse tasks. In particular, \lra is a suite of tasks including Listops output prediction \citep{nangia2018listops}, byte-level text classification on IMDb \citep{maas2011imdb}, byte-level document retrieval on AAN \citep{radev2013aan}, pixel-level image recognition on CIFAR-10 \citep{krizhevsky2009cifar} and Pathfinder \citep{linsley2018pathfinder}. We follow the experimental setup in \citet{xiong2021nystromformer,chen2021skyformer} and adopt the \emph{same} hyper-parameter setting across all attention variants to ensure a fair comparison. In particular, all tasks use a 2-layer Transformer model with 64 embedding dimension, 128 hidden dimension in feed forward neural networks and 2 attention heads. The transformer output is then aggregated by mean pooling (instead of class tokens) for task-specific prediction. The training details for each task are the same for all attention methods as in \citet{xiong2021nystromformer}. For baselines, we compare our model against the standard softmax attention and Performer \citep{choromanski2021rethinking} as well as other efficient attention mechanisms, including Nystr\"omformer \citep{xiong2021nystromformer}, Linformer \citep{wang2020linformer}, Reformer \citep{Kitaev2020reformer} and BigBird \citep{zaheer2020bigbird}.

\begin{table}[t]
	\caption{Top-1 classification accuracy results (\%) on \lra benchmark with different attention mechanisms.}
	\label{tb:lra-diff-attn}
	\vskip 0.15in
	\centering
	\resizebox{0.6\columnwidth}{!}{    
		\begin{tabular}{lcccccc}
            \toprule[.1em] 
            Model & ListOps & Text & Retrieval & Image & Pathfinder & Avg. \\
            \midrule 
            Softmax & 38.76 & 64.90 & 80.54 & 39.90 & 71.29 & 59.08 \\
            Nystr\"omformer & 38.26 & 64.00 & 80.57 & 40.07 & 68.47 & 58.27\\
            Linformer & 37.40 & 59.10 & 78.04 & 38.25 & 60.09 & 54.58 \\
            Reformer & 37.60 & 64.15 & 79.18 & 43.57 & 66.44 & 58.19 \\
            BigBird & 38.81 & 64.02 & 80.73 & 38.56 & 71.60 & 58.74 \\
            Performer & 37.20 & 64.73 & 79.91 & 37.64 & 68.68 & 57.63 \\
            \midrule
            LARA & 39.21 & 64.77 & 81.18 & 38.40 & 72.02 & 59.12 \\
            RA & 38.56 & 65.02 & 80.93 & 40.76 & 71.22 & \textbf{59.30} \\
            \bottomrule[.1em]
\end{tabular}}  
	\vskip -0.1in
\end{table}

As shown in \cref{tb:lra-diff-attn}, we see that RA performs better than softmax attention on 3 out of 5 tasks and obtains a higher averaged accuracy. Furthermore, its linear-complexity counterpart LARA also performs competitively with or slightly outperforms softmax attention except the image task. Both RA and LARA yield better performance than Performer and other baselines on all of 5 tasks, indicating the improved expressiveness of our proposed method. As the sequence length considered in this suite of tasks is typically longer, these results also validates the ability of RA and LARA to capture longer-term dependencies. 

\subsection{Ablation Study}
\label{app:ssec:ablation_study}
\begin{table}
	\caption{Ablation study of LARA, evaluated on \imagenet validation set with DeiT-Tiny model.}
	\label{app:tb:ablation_study}
	\vskip 0.15in
	\centering
	\resizebox{0.7\columnwidth}{!}{    
		\begin{tabular}{l | l | c}
			\toprule[.1em]
			{Components} & {Variants} & {Top-1 Acc.} \\
			\midrule
			\multicolumn{2}{c|}{Performer} & 65.92\\
			\midrule
			\multirow{2}{*}{Single or multiple proposals} & Single proposal & 68.42\\
			& Multiple proposals & 71.48\\
			\midrule
			\multirow{3}{*}{Weighting functions} & Balance heuristic  & 70.78 \\
			& Approx. optimal (coupled; \cref{app:eqn:weighting_coupled})  & 71.02 \\
			& Approx. optimal (decoupled; \cref{app:eqn:weighting_decoupled}) & 71.48 \\
			\midrule
			\multirow{4}{*}{Parameterization of each proposal} & Gaussian mixture (\cref{app:eqn:q_c_as_a_gmm}) & 70.22 \\
			& Gaussian (\cref{app:eqn:q_c_as_a_gaussian}) & 71.22 \\
			& Gaussian (\cref{app:eqn:q_c_as_local_means}) & 71.19 \\
			& Gaussian (\cref{app:eqn:q_c_as_mixed}) & 71.48\\
			\bottomrule[.1em]
	\end{tabular}}  
	\vskip -0.1in
\end{table}
In this section, we conduct an ablation study on vision transformers with \imagenet dataset to investigate the effect of component design in LARA. The main component design choices in LARA consist of the estimation framework (single proposal versus multiple proposals), the parameterization of proposal distributions (Gaussian mixtures versus Gaussian) and the weighting functions. The results are shown in \cref{app:tb:ablation_study}. For the estimation framework, we compare our choice, which uses multiple proposal distributions, against a single proposal. This proposal is a Gaussian mixture with the similar formulation of true RA density (\cref{eqn:ra-density}) except that it only depends on the average of all queries. We see that an individual yet contextual proposal improves the performance of Performer, while generalizing the importance sampling in RFA to multiple proposals further boosts performance to be close to softmax attention. With multiple proposal distributions, even using a simple strategy (balance heuristic \citep{veach1995optimally}) to combine their estimates yields reasonable performance, which is improved further by adopting query-specific combinations. In addition, we validate the effectiveness of decoupling the effect of query-dependent and query-agnostic information inside the weighting function, which improves the coupled version by over 0.4 accuracy. In terms of the parameterization of each proposal distribution, we consider both the cases where each proposal is a Gaussian mixture and a Gaussian. As specified in \cref{app:ssec:lara}, we train the transformer model with various parameterization choices (defined by \cref{app:eqn:q_c_as_a_gmm} for Gaussian mixtures and \cref{app:eqn:q_c_as_a_gaussian,app:eqn:q_c_as_local_means,app:eqn:q_c_as_mixed} for Gaussians). The results are consistent with the analysis in \cref{app:ssec:lara}, where a simple parameterization suffices to yield good performance.

%%%%%%%%%%%%%%%%%%%%%%%%%%%%%%%%%%%%%%%%%%%%%%%%%%%%%%%%%%%%%%%%%%%%%%%%%%%%%%%
%%%%%%%%%%%%%%%%%%%%%%%%%%%%%%%%%%%%%%%%%%%%%%%%%%%%%%%%%%%%%%%%%%%%%%%%%%%%%%%
\end{document}